%% file: main.tex
\pdfoutput=1
\documentclass[10pt,letterpaper]{article}
\usepackage[T1]{fontenc}
\usepackage[utf8]{inputenc}
\usepackage{amsmath,amssymb,amsthm}
\usepackage{newtxtext,newtxmath}
\usepackage[letterpaper,text={5.5in,8.85in},top=1.15in,headheight=13pt,headsep=0.2in,footskip=0.4in]{geometry}
\usepackage{microtype}
\usepackage{graphicx}
\usepackage{booktabs,array,tabularx}
\usepackage[export]{adjustbox}
\usepackage[table]{xcolor}
\usepackage{url}
\usepackage[round,authoryear]{natbib}
\usepackage[explicit]{titlesec}
\usepackage{fancyhdr}
\usepackage[font=small,labelfont=normal,labelsep=colon,justification=justified,singlelinecheck=false,
            skip=6pt,position=bottom]{caption}
\usepackage[most]{tcolorbox}
\usepackage{enumitem}
\usepackage{aliascnt}
\usepackage{float}
\usepackage[colorlinks=true,linkcolor={red!72!black},citecolor={blue!45!black},urlcolor={blue!45!black},
            breaklinks=true]{hyperref}
\usepackage[capitalise,nameinlink]{cleveref}

\setlist[itemize]{leftmargin=1.4em,itemsep=1pt,topsep=2pt}

\titleformat{\section}{\fontsize{11}{13}\selectfont\scshape}{\thesection}{0.9em}{#1}
\titlespacing*{\section}{0pt}{17pt plus 3pt minus 2pt}{7pt plus 1pt}
\titleformat{\subsection}{\normalsize\scshape}{\thesubsection}{0.9em}{#1}
\titlespacing*{\subsection}{0pt}{13pt plus 2pt minus 2pt}{5pt plus 1pt}
\titleformat{\paragraph}[runin]{\bfseries}{}{0pt}{#1.}
\titlespacing*{\paragraph}{0pt}{3pt}{0.7em}
\renewcommand{\thesubsection}{\thesection.\arabic{subsection}}

\makeatletter
\newcommand{\defV}[2]{\expandafter\def\csname V@#1\endcsname{#2}}
\newcommand{\V}[1]{\ifcsname V@#1\endcsname\csname V@#1\endcsname\else\pending{#1}\fi}
\makeatother
\newcommand{\pending}[1]{\fcolorbox{red}{yellow!40}{\scriptsize\ttfamily PENDING:\,\detokenize{#1}}}
\input{numbers}

\makeatletter
\def\capsplit#1. #2\@nil{\textbf{#1.} #2}
\newcommand{\figwide}[4][!tb]{\begin{figure}[#1]\centering
  \IfFileExists{figures/#2.pdf}{\includegraphics[max width=\textwidth]{figures/#2}}{\pending{figure #2}}
  \caption{\capsplit #3 \@nil}\label{#4}\end{figure}}
\makeatother
\newcommand{\figcol}[4][t]{\figwide[#1]{#2}{#3}{#4}}
\newcommand{\Aref}[1]{\hyperref[#1]{Appendix~\ref*{#1}}}
\newcommand{\tab}[1]{\IfFileExists{tables/#1.tex}{\input{tables/#1}}{\pending{table #1}}}

\newtheoremstyle{flood}{0pt}{0pt}{\upshape}{}{\bfseries}{.}{0.5em}%
  {\thmname{#1}\thmnumber{ #2}\textnormal{\thmnote{ (#3)}}}
\theoremstyle{flood}

\newaliascnt{lemma}{theorem}
\newtheorem{lemma}[lemma]{Lemma}
\aliascntresetthe{lemma}
\newaliascnt{proposition}{theorem}
\newtheorem{proposition}[proposition]{Proposition}
\aliascntresetthe{proposition}
\newaliascnt{corollary}{theorem}

\aliascntresetthe{corollary}
\newaliascnt{definition}{theorem}

\aliascntresetthe{definition}
\tcbset{floodbox/.style={enhanced jigsaw,sharp corners,frame hidden,boxrule=0pt,colback=black!8,
  left=6pt,right=6pt,top=4pt,bottom=3pt,boxsep=0pt,before skip=9pt,after skip=9pt,
  grow to left by=4pt,grow to right by=4pt}}
\tcolorboxenvironment{theorem}{floodbox}
\tcolorboxenvironment{lemma}{floodbox}
\tcolorboxenvironment{proposition}{floodbox}
\tcolorboxenvironment{corollary}{floodbox}
\tcolorboxenvironment{definition}{floodbox}

\crefname{proposition}{Proposition}{Propositions}
\crefname{lemma}{Lemma}{Lemmas}
\crefname{appendix}{Appendix}{Appendices}
\crefname{figure}{Figure}{Figures}
\crefname{table}{Table}{Tables}
\crefname{section}{Section}{Sections}

\newtcolorbox{questionbox}{enhanced,arc=3.5pt,boxrule=0.5pt,colframe=black!25,colback=black!3,
  left=10pt,right=10pt,top=5pt,bottom=5pt,before skip=8pt,after skip=6pt,width=0.94\textwidth,center}

\renewenvironment{abstract}{%
  \begin{center}{\fontsize{11}{13}\selectfont\scshape Abstract}\end{center}\vspace{-6pt}%
  \list{}{\setlength{\leftmargin}{0.5in}\setlength{\rightmargin}{0.5in}\setlength{\listparindent}{0pt}\setlength{\parsep}{0pt}}%
  \item\relax}{\endlist}

\newcommand{\pp}{\,pp}

\begin{document}

{\raggedright\fontsize{15}{18}\selectfont\scshape
Allocation Before the Clock: Graph, Set and Game Theory of Professional League of Legends Objective Setups\par}
\vspace{6pt}
{\bfseries Zheng Gao\textsuperscript{$\dagger$}, Zhenyu Liu, Yifan Yang, Xiaoyu Li}\\
UNSW Sydney \quad {\small\textsuperscript{$\dagger$}Corresponding author:} {\small\texttt{zheng.gao1@unsw.edu.au}}
\vspace{2pt}

\input{sections/00_abstract}

\begin{questionbox}
\textbf{The question.}\ \emph{When a neutral objective is about to spawn on a clock that both teams can read, how does each team allocate its five players on a walled map, and what do teams actually price when they decide how many to commit?}
\end{questionbox}

\input{sections/01_intro}
\input{sections/02_data}
\input{sections/03_measurement}
\input{sections/04_model}
\input{sections/05_results}
\input{sections/06_related}
\input{sections/07_limitations}
\input{sections/08_conclusion}

\section*{Acknowledgements}
This research includes computations using the computational cluster Katana supported by Research Technology Services at UNSW Sydney \citep{katana}.

{\small
\setlength{\bibsep}{2.5pt plus 0.5pt}
\bibliographystyle{plainnat}
\bibliography{refs}
}

\clearpage
\appendix
\input{sections/A_appendix}

\end{document}

%% file: numbers.tex
\defV{ameCells}{48}
\defV{ameGoldFive}{+0.024}
\defV{ameGoldFiveCI}{[0.018, 0.033]}
\defV{ameGoldFour}{$-$0.015}
\defV{ameGoldZero}{$-$0.012}
\defV{ameSigCount}{24}
\defV{ameSoulFour}{+0.030}
\defV{ameSoulFourCI}{[0.013, 0.045]}
\defV{auditPooled}{$-$4.1}
\defV{auditPooledP}{0.052}
\defV{auditTagnet}{$-$6.4}
\defV{auditTagnetCI}{[$-$11.2, $-$0.1]}
\defV{auditTagnetP}{0.046}
\defV{auditWithin}{$-$6.1}
\defV{auditWithinP}{0.078}
\defV{cleanCorr}{0.82}
\defV{clockEndDiffMed}{$-$4.2}
\defV{clockGames}{166}
\defV{clockLiveMedPct}{79.2}
\defV{clockMisreadMedPct}{0.06}
\defV{clockReplaySeg}{23}
\defV{fakerDpit}{$-$469}
\defV{fakerDpitCI}{[$-$709, $-$275]}
\defV{fakerN}{432}
\defV{fakerOthersN}{2,545}
\defV{fakerRank}{$-$0.12}
\defV{fakerRankCI}{[$-$0.23, $-$0.03]}
\defV{formExact}{0.538}
\defV{formRand}{0.849}
\defV{gapChannels}{+0.5}
\defV{gapModel}{$-$4.7}
\defV{gapModelCI}{[$-$6.7, $-$4.0]}
\defV{gapNone}{$-$5.1}
\defV{gapObs}{$-$19.2}
\defV{gapObsCI}{[$-$23.9, $-$15.1]}
\defV{gapShare}{0.24}
\defV{hsBootKSix}{0.819}
\defV{hsBootKSixN}{27}
\defV{hsBootKSixQ}{0.756}
\defV{hsCrossKSix}{0.890}
\defV{hvStepFiveSix}{0.82}
\defV{hvStepSevenEight}{0.37}
\defV{hvStepSixSeven}{0.41}
\defV{hvTestK1}{$-$92.85}
\defV{hvTestK6}{$-$80.00}
\defV{hvTestK8}{$-$79.21}
\defV{lamGenG}{3.7}
\defV{lamGenGCI}{[2.2, 7.3]}
\defV{lamObjTOne}{0.060}
\defV{lamObjTOneCI}{[0.032, 0.145]}
\defV{lamTOne}{3.3}
\defV{lamTOneCI}{[1.8, 6.3]}
\defV{latticeDiam}{4}
\defV{latticeEdges}{160}
\defV{latticeN}{52}
\defV{latticeObservedAll}{52}
\defV{latticeShapes}{7}
\defV{ledgerSpawns}{51}
\defV{lkHasseCut}{0.330}
\defV{lkHasseLeaky}{0.331}
\defV{monoDiff}{0.027}
\defV{monoFolds}{4}
\defV{monoNats}{197}
\defV{monoSD}{0.050}
\defV{nGames}{164}
\defV{nSeries}{60}
\defV{nSetups}{2,994}
\defV{nSpawns}{1,516}
\defV{nTeams}{19}
\defV{nocostLL}{$-$9,254}
\defV{othersTL}{$-$0.192}
\defV{patFiveIota}{0.65}
\defV{patFiveTheta}{0.73}
\defV{patFiveWLate}{0.27}
\defV{patFourTheta}{1.17}
\defV{patFourWEarly}{0.19}
\defV{patOneAtSpawn}{BS|T|J|M}
\defV{patOneTheta}{1.28}
\defV{patOneWEarly}{0.25}
\defV{patThreeAtSpawn}{JMBS|T}
\defV{patThreeIota}{0.72}
\defV{patThreeTheta}{0.56}
\defV{patThreeWLate}{0.38}
\defV{pilotDragonLead}{43}
\defV{pilotDragonLeadN}{21}
\defV{pilotDragonTrail}{13}
\defV{qContests}{1,454}
\defV{qConvN}{170}
\defV{qConvShapDiff}{0.2}
\defV{qConvShare}{57}
\defV{qKappa}{1.79}
\defV{qLL}{$-$8,343}
\defV{qLLsat}{$-$6,826}
\defV{qLam}{4.30}
\defV{qLamCI}{[3.01, 6.46]}
\defV{qLamOnePoint}{4.4}
\defV{qNu}{0.018}
\defV{qNuCI}{[0.013, 0.026]}
\defV{qNuConv}{0.018}
\defV{qNuNonconv}{0.019}
\defV{qOmegaK}{0.185}
\defV{qOmegaM}{0.69}
\defV{qPsi}{1.81}
\defV{qReps}{300}
\defV{qWinmax}{62.9}
\defV{qWinmaxCI}{[40.1, 92.8]}
\defV{qWinmaxLR}{120}
\defV{qWinmaxMin}{32.8}
\defV{qWinmaxPoint}{59.9}
\defV{rAtakhanEmptyHi}{22}
\defV{rAtakhanEmptyLo}{16}
\defV{rAtakhanQ}{0.08}
\defV{rBaronEmpty}{+19.7}
\defV{rBaronEmptyCI}{[11.4, 29.9]}
\defV{rBaronMeanK}{$-$0.39}
\defV{rBaronMeanKCI}{[$-$0.64, $-$0.14]}
\defV{rDragonEmptyAtSpawnMin}{95}
\defV{rDragonSixtyEmpty}{+9.1}
\defV{rDragonSixtyEmptyCI}{[$-$5.6, 22.8]}
\defV{raceAtakhan}{0.51}
\defV{raceAtakhanN}{106}
\defV{raceBaron}{0.68}
\defV{raceDragonTwo}{0.80}
\defV{raceDragonTwoN}{264}
\defV{raceHerald}{0.77}
\defV{raceMedianT}{30}
\defV{sSevenBn}{$-$0.57}
\defV{sSevenBnCI}{[$-$0.84, $-$0.29]}
\defV{sSevenBnVfive}{$-$2.33}
\defV{sSevenBnVfiveCI}{[$-$2.70, $-$1.96]}
\defV{sSevenFiveThree}{0.037}
\defV{sSevenFiveThreeSq}{0.375}
\defV{sSevenTwoOne}{0.132}
\defV{sSevenTwoOneSq}{0.333}
\defV{setupsTeamMax}{394}
\defV{setupsTeamMin}{31}
\defV{shB}{$-$1.51}
\defV{shBCI}{[$-$3.42, $-$1.04]}
\defV{shC}{+1.50}
\defV{shCCI}{[1.14, 2.10]}
\defV{shCminusL}{+1.39}
\defV{shCminusLCI}{[1.02, 1.91]}
\defV{shF}{+0.30}
\defV{shL}{+0.15}
\defV{shMax}{0.35}
\defV{shMin}{$-$0.04}
\defV{shZeroIn}{5}
\defV{simpMatched}{+0.240}
\defV{simpMatchedCI}{[$-$0.354, 0.778]}
\defV{simpRaw}{$-$0.192}
\defV{simpRawCI}{[$-$0.237, $-$0.147]}
\defV{singleMid}{29.7}
\defV{singleMidCI}{[24.1, 35.1]}
\defV{singleTop}{49.5}
\defV{singleTopCI}{[44.1, 54.9]}
\defV{synthAri}{0.915}
\defV{synthAriMissing}{0.903}
\defV{synthIotaErr}{0.08}
\defV{synthLLspread}{0.0}
\defV{synthN}{600}
\defV{synthPathsOK}{3 of 3}
\defV{synthThetaErr}{0.10}
\defV{tOneFightRes}{$-$0.002}
\defV{tOneFightResCI}{[$-$0.045, 0.046]}
\defV{tOneFights}{399}
\defV{tOneMSI}{0.203}
\defV{tOneMinusOthers}{$-$0.002}
\defV{tOneMinusOthersCI}{[$-$0.133, 0.149]}
\defV{tOneNearest}{Gen.G}
\defV{tOneNearestP}{40}
\defV{tOneNearestTV}{0.032}
\defV{tOneSecond}{Top Esports}
\defV{tOneSetups}{394}
\defV{tOneTL}{$-$0.194}
\defV{tOneTLCI}{[$-$0.320, $-$0.064]}
\defV{tOneThird}{KT Rolster}
\defV{tOneUsage}{0.21/0.20/0.16/0.13/0.09/0.21}
\defV{tOneWorlds}{0.254}
\defV{tgBootAri}{0.813}
\defV{tgCountAri}{0.086}
\defV{tgLeakRatio}{1.24}
\defV{tgN}{2,510}
\defV{tgSeedAri}{0.421}
\defV{tgSeedAriCtx}{0.319}
\defV{tgSilSix}{0.121}
\defV{tgVgold}{0.343}
\defV{tgVgoldCtx}{0.425}
\defV{trade}{+0.073}
\defV{tradeCI}{[$-$0.026, 0.175]}
\defV{trkCrowdFourLtSix}{0.493}
\defV{trkCrowdZeroLtSix}{0.905}
\defV{trkDetLtSix}{0.900}
\defV{trkDetMed}{146}
\defV{trkDetShare}{78}
\defV{trkInterpMed}{921}
\defV{trkInterpShare}{21}
\defV{trkLtSix}{0.785}
\defV{trkMed}{177}
\defV{trkN}{23,286}
\defV{trkPninetyfive}{7,582}
\defV{tulFiveFourGoldEq}{1.6}
\defV{tulGg}{10.9}
\defV{tulGgCI}{[9.2, 12.8]}
\defV{tulGn}{0.79}
\defV{tulGnCI}{[0.63, 0.95]}
\defV{tulGnN}{1,481}
\defV{tulGnVfive}{4.72}
\defV{tulRatio}{0.073}
\defV{tulSide}{$-$0.029}
\defV{tulSideCI}{[$-$0.154, 0.093]}
\defV{vCAtakhanFortyfive}{+30.4}
\defV{vCAtakhanFortyfiveCI}{[15.0, 46.5]}
\defV{vCAtakhanFortyfiveMeanK}{$-$0.67}
\defV{vCAtakhanFortyfiveMeanKCI}{[$-$1.02, $-$0.32]}
\defV{vCBaronFortyfive}{+34.8}
\defV{vCBaronFortyfiveCI}{[18.7, 50.0]}
\defV{vCDragonTwoFifteen}{+33.7}
\defV{vCDragonTwoFifteenCI}{[24.9, 41.5]}
\defV{vCDragonTwoFifteenSelf}{$-$10.9}
\defV{vCDragonTwoFifteenSelfCI}{[$-$15.7, $-$5.5]}
\defV{vCHeraldMinusSeventyfive}{+12.3}
\defV{vCHeraldMinusSeventyfiveCI}{[3.4, 21.0]}
\defV{vContrastCells}{300}
\defV{vGridSnaps}{21}
\defV{vHeadlineCells}{53}
\defV{vHeadlineCleanKept}{30}
\defV{vHeadlineCleanSign}{42}
\defV{vHeadlineTrailLess}{53}
\defV{vSnapMedPct}{90}
\defV{vSpawnsCovered}{1,497}
\defV{valATAKHANWpa}{20.5}
\defV{valBARONNASHORWpa}{32.6}
\defV{valDRAGONOneWpa}{6.0}
\defV{valDRAGONThreeWpa}{21.6}
\defV{valFIGHTWpa}{8.7}
\defV{varShapMax}{2.4}
\defV{wCalMaxDevPP}{1.3}
\defV{xiEight}{0.221}
\defV{xiFive}{0.127}
\defV{xiSix}{0.137}

%% file: sections/00_abstract.tex
\begin{abstract}
Before every neutral objective in professional \emph{League of Legends}, each team decides which of its five players walk to the pit, which hold a lane, and when to commit.
The leading analysis of such setups, TAGNet for \emph{Counter-Strike~2}, clusters graph-network embeddings of player positions and reads cluster win rates as tactical quality.
We treat a setup as an allocation on a walled map: formations are set partitions on the lattice $\Pi_5$, control is a geodesic Voronoi split of a navigation graph, and commitment is an allocation game solved as a quantal response equilibrium priced in win probability.
On all \V{nGames} games of MSI and Worlds 2025 (\V{nSetups} team-setups, minimap tracks validated against Riot positions), trailing teams leave the pit empty far more often than leading teams, by \V{vCBaronFortyfive}\pp{} at the first baron and \V{vCAtakhanFortyfive}\pp{} at Atakhan.
Our model, HASSE-QRE, finds six setup patterns that whole-model series-bootstrap refits reproduce (ARI \V{hsBootKSix}, comparable to \V{tgBootAri} for TAGNet's $K$-Means step alone) and indicates that teams price fights, travel and committed players, not the objective itself (weight \V{qNu}; fixing it to one costs \V{qWinmaxPoint} log-likelihood points).
The equilibrium reproduces a quarter of the commitment gap, almost all of it through state variables rather than its strategic channels.
Weighting a team-level contrast by exposure moves it across $p=0.05$ (\V{auditTagnetP} against \V{auditPooledP}).
We release the tracks, setups and code.
\par\smallskip\noindent\textit{Keywords:} esports analytics, League of Legends, set partitions, quantal response equilibrium, Lanchester laws, Colonel Blotto, spatiotemporal tracking
\end{abstract}

%% file: sections/01_intro.tex
\figwide[H]{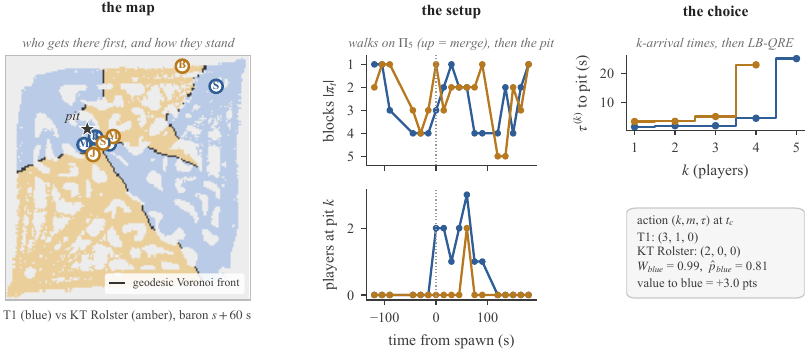}{The paper in one picture. One real setup (game~4 of the Worlds~2025 final, first baron, 60\,s after the spawn). \emph{The map:} walkable graph (white: walls) coloured by which team reaches each cell first, with each formation drawn as the blocks of a set partition. \emph{The setup:} each team's walk on the partition lattice $\Pi_5$ (up = fewer blocks) and its players at the pit. \emph{The choice:} $k$-arrival times $\tau^{(k)}$ and the LB-QRE action $(k,m,\tau)$ at the commitment time, with win probability $W$, fight probability $\hat p$ and the objective's value.}{fig:teaser}

\section{Introduction}
\label{sec:intro}

Neutral objectives in \emph{League of Legends} (LoL) appear on a public clock.
The first dragon spawns at 5:00, the void grubs at 8:00, the Rift Herald at 15:00, Atakhan at 20:00 and Baron Nashor at 25:00, and both teams know these times to the second.
In the minute before each spawn the two teams move their ten players into position: some walk to the pit, some keep pushing a lane, some wait in the jungle in between.
This setup decides who fights for the objective, with how many players, and whether the other side trades it for towers elsewhere.
Professional coaches plan it in detail. Esports analytics has measured almost none of it.

The closest work is TAGNet \citep{zhang2026tagnet}, which studies how defenders set up after the freeze time in professional \emph{Counter-Strike~2}.
TAGNet embeds the five defenders as a complete graph with a graph attention network \citep{velickovic2018graph} and an untrained recurrent layer, clusters the 64-dimensional embeddings with $K$-Means at a fixed $K=6$, and reads the win rate of each cluster as the quality of a tactic.
Three steps of that pipeline make the conclusions fragile.
The graph carries no map: straight-line distances ignore walls.
The context variables (economy, score) are fed into the representation and later `discovered' to be associated with it.
And comparisons of trailing and leading teams average over teams rather than over rounds.

We take a different view. A setup is an \emph{allocation on a walled map}.
Distances are geodesic on a navigation graph built from Riot's own position data, so a formation is the set of connected components of a proximity graph, that is, a set partition of the five roles, and the formations form the partition lattice $\Pi_5$ with \V{latticeN} elements.
Map control is the geodesic Voronoi split of that graph, and readiness is a vector of $k$-arrival times.
The decision is how many players to commit to the pit, how many to a second front, and whether to be there before the spawn.
Fight outcomes follow a Lanchester-type law \citep{lanchester1916aircraft} written as a Tullock contest \citep{tullock1980efficient}, the second front is a Blotto-type contest \citep{roberson2006colonel}, and every outcome is priced in win probability by a surface $W$ fitted on game results only.
\Cref{fig:teaser} shows all of these objects on one real baron setup from the Worlds~2025 final.

The first finding is about behaviour. Across all \V{nGames} games of MSI and Worlds~2025, trailing teams leave the pit empty far more often than leading teams once the objective is up.
At the first baron the gap reaches \V{vCBaronFortyfive}\pp{} \V{vCBaronFortyfiveCI} at $s{+}45$\,s, at Atakhan \V{vCAtakhanFortyfive}\pp{} \V{vCAtakhanFortyfiveCI}, and at respawned dragons \V{vCDragonTwoFifteen}\pp{} \V{vCDragonTwoFifteenCI} (\Cref{fig:curves,fig:contrasts}).
Every headline contrast keeps its sign at three pit radii and survives false-discovery control.
The same contrast on Riot's exact 60-second frames gives \V{rBaronEmpty}\pp{} \V{rBaronEmptyCI} at the first baron.
The dragon gradient of our kill-anchored pilot (\V{pilotDragonTrail}\% of trailing vs \V{pilotDragonLead}\% of leading non-takers contested) disappears once spawns rather than kills anchor the window: it was an artefact of selecting on the outcome.

The second finding is about structure. HASSE, a sticky hidden Markov model whose emissions live on $\Pi_5$, describes each setup as a walk on the Hasse diagram with a pattern-specific centre path, rigidity and inertia.
Held-out likelihood improves up to $K=8$ with steps that shrink to \V{hvStepSevenEight} nats per setup, and six patterns carry the structure (\Cref{fig:patterns}).
Two of them, used almost only at early river objectives, stay spread out; two gather four players around the pit at the spawn and dominate the late game.
Whole-model series-bootstrap refits of HASSE agree at ARI \V{hsBootKSix} (\V{hsBootKSixN} series-bootstrap resamples), comparable to re-running only TAGNet's $K$-Means step on fixed embeddings (\V{tgBootAri}); five independent HASSE fits, each the best of ten restarts, agree at \V{hsCrossKSix}, whereas single TAGNet runs on the same V-layer tracks agree at \V{tgSeedAri} across seeds and at \V{tgCountAri} with TAGNet's own counting baseline (\Cref{fig:stability}).

The third finding is about incentives. LB-QRE, a heterogeneous logit quantal response equilibrium \citep{mckelvey1995quantal} of the allocation game estimated by nested pseudo-likelihood \citep{aguirregabiria2007sequential}, finds that teams respond to fight values, travel and the cost of each committed player, and barely to the objective's own value: its weight is $\nu=\V{qNu}$ \V{qNuCI}, and fixing $\nu=1$ (pure win-probability maximisation) costs \V{qWinmaxPoint} log-likelihood points (likelihood ratio about \V{qWinmaxLR}; bootstrap mean \V{qWinmax} \V{qWinmaxCI}).
The equilibrium reproduces about a quarter of the observed trailing-minus-leading commitment gap (\V{gapModel} vs \V{gapObs}\pp), but almost all of that comes from state variables (alive counts, the gold term, travel and feasibility): with every strategic channel off the gap is \V{gapNone}\pp, so the four channels add \V{gapChannels}\pp{} net.
Within the model, the curvature of $W$ pushes trailing teams towards contesting (\V{shC}\pp) and team behaviour pulls them away (\V{shB}\pp), while the Lanchester and second-front channels are small (\Cref{fig:qre}).
Fights punish the smaller side hard: the larger side loses the next player far less often than the square law predicts (\V{sSevenFiveThree} vs \V{sSevenFiveThreeSq} at 5v3).

These results rest on a measurement layer that we validate end to end: a broadcast-clock reader aligns every video to game time, the minimap is registered to the navigation graph, and a per-game tracker is scored against Riot positions at held-out minutes (median error \V{trkMed} map units on a map 14,820 units wide).
In summary, we contribute
(i)~a theory-defined measurement of objective setups (navigation graph, geodesic Voronoi control, $k$-arrival order statistics, partitions on $\Pi_5$) with validated extraction from broadcast video;
(ii)~HASSE-QRE, a two-layer model whose discovery layer lives on the partition lattice and whose explanation layer is an estimable allocation game in win-probability units;
(iii)~the findings above, including the negative ones, with an estimator audit that shows how the choice of estimator can flip significance; and
(iv)~an open dataset of \V{nSpawns} spawns, \V{nSetups} team-setups and 2-fps tracks, with an appendix that turns narratives about T1 and Faker into measured quantities.

%% file: sections/02_data.tex
\section{Data and the Spawn Clock}
\label{sec:data}

\input{tables/tab_data}

\paragraph*{Games}
The core set is every game of the 2025 Mid-Season Invitational and the 2025 World Championship, \V{nGames} games in \V{nSeries} series between \V{nTeams} teams (\Cref{tab:data}).
Both events were played on patches 25.13 and 25.20 under the same objective rules, so one spawn clock covers all games.
For each game we hold three sources: the broadcast video, Riot's match-v5 timeline \citep{riotapi} with exact positions of all ten players every 60\,s and millisecond-stamped kills and objectives, and Leaguepedia \citep{leaguepedia} metadata (roles, draft, series).
An archive of 2,678 professional games with v5 timelines (2023--2026) is used only to fit outcome models: the win surface, event values and the navigation graph.
It never sees a setup choice.

\paragraph*{The spawn clock}
\Cref{tab:rules} lists the rules.
We infer every spawn time from the rules and the kill events and check that no kill precedes its inferred spawn; none does across the \V{nSpawns} spawns.
Anchoring at the spawn rather than at the kill is the central design choice of this paper.
The spawn time is fixed by rule and known to both teams, so a window anchored there does not depend on what the teams did; a window anchored at the kill selects on who took the objective and when.
After the spawn we keep a team-snapshot only while the objective is still alive (a risk set), because a snapshot taken after the kill describes the aftermath rather than the setup.

\input{tables/tab_rules}

\paragraph*{Units}
A \emph{team-setup} is one team at one spawn, \V{nSetups} in total.
We observe it on a grid of snapshots every 15\,s from $s{-}120$ to $s{+}180$\,s, where $s$ is the spawn time.
Two layers share these units.
The \textbf{V layer} uses 2-fps minimap tracks from the broadcast (\Cref{sec:measurement}) and covers \V{vSpawnsCovered} of \V{nSpawns} spawns (the rest lack a live track), \V{nSetups} team-setups, on a grid of \V{vGridSnaps} snapshots of which a median \V{vSnapMedPct}\% are observed; missing snapshots are treated as missing.
The \textbf{R layer} uses only Riot's exact 60-s frames and covers the fixed-time first spawns (dragon, grubs, herald, Atakhan, baron) at $s{-}60$, $s$ and $s{+}60$.
The R layer has no measurement error and coarse time; the V layer has fine time and validated measurement error.
Findings that matter are checked on both.

\paragraph*{Gold state}
A team is \emph{trailing}, \emph{even} or \emph{leading} by its gold difference at $s{-}120$\,s divided by the archive standard deviation of the gold difference at that game minute, with cut points $\pm 0.5$.
The state is fixed before the setup begins, so it cannot respond to the setup. It is measured at $s{-}120$\,s on the V layer, at $s{-}60$\,s on the R layer (the first exact frame) and, for LB-QRE, at the last Riot frame at least 30\,s before the spawn.

\paragraph*{Objective classes}
Mixture weights in the discovery layer are stratified by three classes that differ in timing and map: early river objectives (dragons, grubs, herald), Atakhan, and late objectives (baron, elder).
In the descriptive analysis respawned dragons and barons form their own groups because their spawn time depends on play.

%% file: tables/tab_data.tex
\begin{table}[t]
\centering
\small
\caption{\textbf{The core set covers every game of two 2025 events.} The core video set. Every game of MSI 2025 and Worlds 2025 under one set of objective rules.}
\label{tab:data}
\setlength{\tabcolsep}{4pt}
\begin{adjustbox}{max width=\linewidth}
\begin{tabular}{lrrrr}
\toprule
Event & Games & Series & Teams & Spawns \\
\midrule
MSI 2025 & 80 & 19 & 10 & 740 \\
Worlds 2025 Main Event & 80 & 40 & 16 & 740 \\
Worlds 2025 Play-In & 4 & 1 & 2 & 36 \\
\midrule
total & 164 & 60 & 19 & 1,516 \\
\bottomrule
\end{tabular}
\end{adjustbox}
\end{table}

%% file: tables/tab_rules.tex
\begin{table}[t]
\centering
\small
\caption{\textbf{Every objective has a fixed spawn clock, and no kill precedes its spawn.} The spawn clock (patch 25.09--25.20 rules) and the resulting units. Delay is the median spawn-to-kill time; the last column gives the share of tracked positions that were detected, interpolated or pinned to a V5 minute.}
\label{tab:rules}
\setlength{\tabcolsep}{4pt}
\begin{adjustbox}{max width=\linewidth}
\begin{tabular}{lllrrrrr}
\toprule
Objective & First spawn & Respawn & Spawns & Taken & Delay (s) & Setups & det/int/pin \\
\midrule
Elemental dragon & 5:00 & kill + 5:00 until a soul & 773 & 743 & 49 & 1,526 & 0.75/0.20/0.06 \\
Elder dragon & soul + 6:00 & kill + 6:00 & 27 & 17 & 44 & 52 & 0.67/0.30/0.00 \\
Void grubs & 8:00 & one camp of 3, no respawn & 164 & 164 & 21 & 322 & 0.71/0.10/0.19 \\
Rift Herald & 15:00 & once & 164 & 164 & 50 & 328 & 0.72/0.16/0.12 \\
Atakhan & 20:00 & once (side set by play) & 164 & 164 & 109 & 328 & 0.71/0.21/0.08 \\
Baron Nashor & 25:00 & kill + 6:00 & 224 & 181 & 117 & 438 & 0.69/0.27/0.03 \\
\bottomrule
\end{tabular}
\end{adjustbox}
\end{table}

%% file: sections/03_measurement.tex
\section{The Measurement Layer}
\label{sec:measurement}

The measurement layer turns video into positions and positions into the objects the models use.
It has no learned parameters beyond the per-game detectors, and every step is validated against Riot data.

\subsection{From broadcast to game coordinates}

\paragraph*{Clock alignment}
We crop the broadcast's game clock at 10\,fps and read it with a digit classifier trained on one game, labelled by a single hand-checked anchor reading and the clock's own progression.
Each clock tick yields one pair (video time, game time); a robust fit over the roughly one thousand ticks of a game gives the offset of each live segment and separates live play from replays, which show earlier times for a few seconds and then return.
All \V{clockGames} videos (the \V{nGames} core games and two control videos) align (\Cref{fig:measurement}a); the median share of frames carrying a live game time is \V{clockLiveMedPct}\%, the median misread rate \V{clockMisreadMedPct}\%, and the last live reading falls a median \V{clockEndDiffMed}\,s before the official game end, the length of the closing camera sequence.
\V{clockReplaySeg} replay segments are detected and excluded.

\paragraph*{Minimap registration}
The minimap is an orthographic projection, so four numbers (scale and offset per axis) map pixels to map coordinates.
The per-pixel median of a game's minimap frames removes the moving icons and leaves the static map; we register it to the walkable mask of the navigation graph by maximising their correlation.
Two independent crops of the same game agree to within 0.04 pixels, and Riot's tower coordinates and per-minute player positions land on their icons.

\paragraph*{Tracking}
Each game has ten known champions, and the v5 timeline gives every champion's exact position once a minute.
We therefore train detectors per game: for each champion a linear discriminant $w=\Sigma^{-1}(\mu_{\mathrm{champ}}-\mu_{\mathrm{bg}})$, where $\Sigma$ is the covariance of random minimap patches, so static structures such as towers fall into the background distribution without labels.
Scores are calibrated to likelihood ratios.
A Viterbi pass per champion chooses among detection peaks, a `held' state (icon occluded, position kept) and a `dead' state that is allowed only inside death windows from the kill events; transitions are priced by geodesic distance against the maximum movement speed.

\figwide{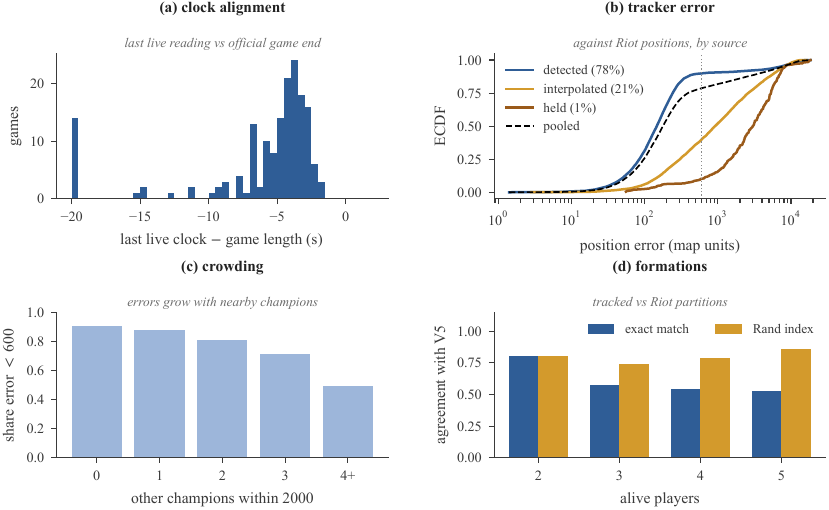}{The measurement layer is validated end to end. (a)~Last live clock reading minus official game length for all aligned videos. (b)~Tracker position error against Riot v5 at held-out odd minutes, by position source; the dotted line marks 600 units. (c)~Share of errors below 600 units by local crowding (other champions within 2,000 units). (d)~Agreement of tracked formations with v5 formations: exact partition match and Rand index, by number of alive players.}{fig:measurement}

\paragraph*{Validation}
We validate on a harder task than the production run: detectors are trained and tracks pinned only at even minutes, and the tracker is scored at odd minutes, a full 60\,s from the nearest pin (production pins every minute, at most 30\,s away).
Over \V{trkN} champion-minutes the median error is \V{trkMed} units and \V{trkLtSix} of errors are below 600 units (\Cref{tab:tracker}, \Cref{fig:measurement}).
Detected positions (\V{trkDetShare}\% of all) have median error \V{trkDetMed} and \V{trkDetLtSix} below 600; interpolated positions (\V{trkInterpShare}\%) have median \V{trkInterpMed}.
Error grows with crowding: the share below 600 falls from \V{trkCrowdZeroLtSix} with no other champion within 2,000 units to \V{trkCrowdFourLtSix} with four or more.
The upper tail (P95 \V{trkPninetyfive}) comes from identity errors in crowds, which is why formation agreement with v5 is \V{formExact} exact and \V{formRand} by Rand index.
The analyses below are built to be robust to this error: contrasts are headlines only if their sign agrees at three pit radii, every V-layer finding has an R-layer counterpart or a sensitivity run without mostly-interpolated snapshots, and the LB-QRE fit is repeated without decisions that rest on held or interpolated positions.

\subsection{Objects defined by theory}

\paragraph*{Navigation graph}
We rasterise the map into 100-unit cells and call a cell walkable if a player stood on it in at least three of 1,994 archive games on the current map (2024--2026, base terrain); changing the threshold to two or five games moves median geodesic distances to the pits by 0--65 units.
The walkable cells form a graph with 16-neighbourhood moves (8-neighbourhoods overestimate distances by up to 8\%), and we precompute all pairwise geodesic distances.
Every distance in this paper is geodesic.

\paragraph*{Formations are set partitions}
Two alive teammates are \emph{together} if their geodesic distance is at most $r=1{,}900$ units, five seconds of movement at the median speed.
The connected components of this proximity graph partition the alive roles; with all five alive the formation is an element of the partition lattice $\Pi_5$ ordered by refinement.
$\Pi_5$ has \V{latticeN} elements in \V{latticeShapes} shapes, its Hasse diagram has \V{latticeEdges} covering edges and diameter \V{latticeDiam} (\Cref{fig:lattice}); all \V{latticeObservedAll} partitions occur in our data.
Increasing $r$ only merges blocks, so the radius sensitivity is a chain in the lattice.
Dead players are removed by restricting the partition to the alive set, never imputed.

\figcol{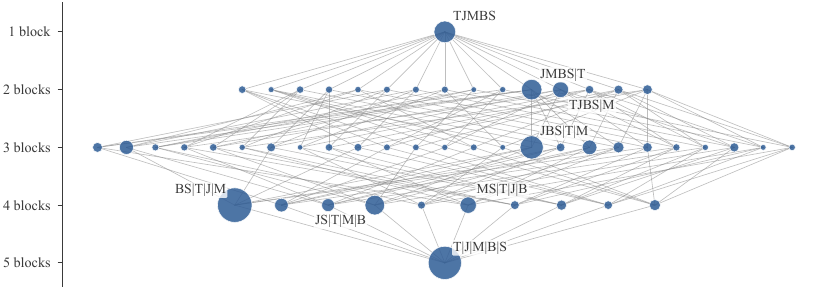}{The partition lattice $\Pi_5$ of the five roles (T top, J jungle, M mid, B bot, S support). Rows are numbers of blocks; edges are covering relations (one merge). Node area is the share of V-layer team-snapshots at the spawn with all five alive.}{fig:lattice}

\begin{lemma}[Lattice metric]\label{lem:metric}
For $\sigma,\tau\in\Pi_5$ let $d(\sigma,\tau)=|\sigma|+|\tau|-2|\sigma\vee\tau|$, where $|\cdot|$ counts blocks and $\vee$ is the join (finest common coarsening).
Then $d$ is a metric on $\Pi_5$, equals the shortest-path length in the Hasse diagram, and every shortest path from $\sigma$ to $\tau$ consists of exactly $|\sigma|-|\sigma\vee\tau|$ merges and $|\tau|-|\sigma\vee\tau|$ splits.
\end{lemma}
The metric is classical for upper semimodular lattices \citep{monjardet1981metrics}; we also verify all three statements exhaustively on $\Pi_5$ (\Aref{app:proofs}).
It lets HASSE count merges and splits in Hasse steps, a unit that a reader can interpret.

\paragraph*{Control and readiness}
For a snapshot, each walkable cell belongs to the team whose nearest alive player reaches it first (a geodesic Voronoi split, left panel of \Cref{fig:teaser}); \emph{pit control} is the team's share of cells within the pit radius.
The \emph{$k$-arrival time} $\tau^{(k)}$ is the $k$-th order statistic of the team's geodesic travel times to the pit (right panel of \Cref{fig:teaser}): the earliest moment at which the team can have $k$ players there.
$k$ is the number of alive players within $R=1{,}500$ units of the pit centre (sensitivity $R\in\{1{,}000,2{,}000\}$).
Each player's \emph{zone} is one of pit, own pit side, opposite pit side, pit lane, mid, far side and base, defined relative to the pit and the team.

%% file: sections/04_model.tex
\section{HASSE-QRE}
\label{sec:model}

HASSE-QRE has two layers with one rule between them.
The discovery layer HASSE describes \emph{what} setups look like and how they unfold; the explanation layer LB-QRE describes \emph{why} teams commit the players they do.
Context variables (gold, side, soul point, team strength) never enter HASSE: they are related to its output only afterwards, a modular `cut' \citep{plummer2015cuts}.
The association between patterns and context therefore cannot be built into the patterns.

\subsection{HASSE: walks on the partition lattice}

A team-setup $e$ is a sequence of observations $y_{et}=(A_t,\pi_t,\{g_B\}_{B\in\pi_t})$ over $T=13$ snapshots ($s{-}120$ to $s{+}60$\,s): the alive set $A_t$, the formation $\pi_t\in\Pi(A_t)$, and the zone $g_B$ of each block.
A pattern $c$ has a \emph{centre path} $\mu_{c,1:T}$ in $\Pi_5$, a rigidity $\theta_c$, an inertia $\iota_c$ and zone preferences $\psi_c$.
Its formation emission is
\begin{equation}
\begin{split}
P_c(\pi_t=\sigma\mid\pi_{t-1})\propto\exp\{&-\theta_c\, d(\sigma,\mu_{c,t}|_{A_t})\\
&-\iota_c\, d(\sigma|_C,\pi_{t-1}|_C)\},
\end{split}
\label{eq:hasse}
\end{equation}
with $d$ the lattice metric of \Cref{lem:metric}, $\cdot|_A$ restriction to the alive set and $C=A_t\cap A_{t-1}$.
The normaliser is an exact sum over at most \V{latticeN} partitions.
$\theta_c$ says how tightly the pattern holds its centre and $\iota_c$ how reluctant it is to merge or split between snapshots; both are measured per Hasse step.
Each block chooses a zone by a product of role experts, $P(g_B=h)\propto\exp\sum_{r\in B}\psi_c[r,h]$, with a linear time trend.
Patterns follow a sticky Markov chain \citep{fox2011sticky}: $z_1\sim\mathrm{Cat}(w_{\mathrm{class}(e)})$ and $P(z_t=c'\mid z_{t-1}=c)=(1-\xi)\mathbb{1}[c'=c]+\xi q_{cc'}$; $\xi=0$ recovers TAGNet's one-label-per-setup clustering.
Snapshots after the kill are missing, not zero.
We fit by EM with exact forward--backward, an exact Viterbi over the 52 centres for each $\mu_c$, and a concave two-dimensional update for $(\theta_c,\iota_c)$ (\Aref{app:hasse}).
We examine $K=1,\dots,8$ by held-out log-likelihood over five series-grouped folds, BIC and ICL \citep{biernacki2000assessing}; no finite $K$ is selected (\Cref{sec:patterns}), so we report $K=6$ and show $K=5$ and $K=8$.
The restricted model $H_{\mathrm{mono}}$ (run at $K=5$ with five restarts) allows the centre path only to merge, i.e.\ setups that only gather.

\subsection{LB-QRE: the allocation game}

\paragraph*{Win surface and values}
$W_A(x)$ is the probability that team $A$ wins from state $x$. We fit $\mathrm{logit}\,W_A(x)=\beta(t)^\top[\phi(x_A)-\phi(x_B)]+\beta_{\mathrm{side}}(t)s_A$ with $s_A=\pm1$, which makes $W_B=1-W_A$ exact. $\phi$ contains gold (hinge basis, so curvature in gold is free), experience, kills, towers, inhibitors, plates, dragons, soul and soul point, grubs, heralds, Atakhans, barons, baron and elder buffs and alive players; $\beta(t)$ varies with the game minute through a B-spline basis, and objective coefficients are specific to the rule set. The fit uses 2,678 archive games, an $L_2$ penalty chosen by series-grouped cross-validation and out-of-fold predictions throughout.
$W$ is calibrated to within \V{wCalMaxDevPP}\pp{} in every decile (\Cref{fig:values}a).
Evaluated 15\,s after the kill, $W$ assigns the baron a \emph{negative} value: $W$ conditions on gold, towers and buffs, and right after the kill none of what the baron will produce has happened yet.
We therefore value every event by the win probability it adds over the next 180\,s, $\mathrm{logit}\,W_{t+180}=a\,\mathrm{logit}\,W_t+\beta(t)\,\mathrm{dir}$, fitted on the archive (\Cref{tab:values}).
At $W=0.5$ a first baron is worth \V{valBARONNASHORWpa} win-probability points, Atakhan \V{valATAKHANWpa}, a soul-point dragon \V{valDRAGONThreeWpa} and the first dragon \V{valDRAGONOneWpa}; a won fight is worth \V{valFIGHTWpa} points per kill.
Total effects on the final result agree (\Cref{fig:values}c).

\figwide{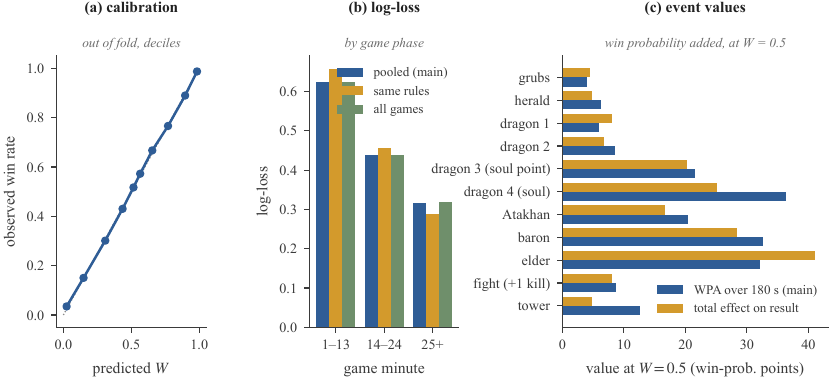}{Win probability prices every outcome. (a)~Calibration of the outcome-only win surface $W$ (out of fold, deciles). (b)~Log-loss by game phase for three specifications. (c)~Event values at $W=0.5$: win probability added over 180\,s (main) and the total effect on the result.}{fig:values}

\paragraph*{Actions}
At the commitment time $t_c$, the first at-risk snapshot in $[s,s{+}180]$ at which some team has at least three players at the pit (otherwise the last at-risk snapshot), team $i$ with $n_i$ alive players has chosen $a_i=(k_i,m_i,\tau_i)$: $k_i$ players at the pit, $m_i$ on a second front (zone `far'), and $\tau_i=1$ if the pit group was already in place at the spawn.
With five alive there are 36 actions; players who cannot reach the pit within 60\,s are infeasible for $k$.

\paragraph*{Outcomes}
If both teams commit, team $i$ wins the fight with
\begin{equation}
p_i=\Lambda\big(\gamma_n\ln(k_i/k_j)+\gamma_g\ln(\bar g_i/\bar g_j)\big),
\label{eq:tullock}
\end{equation}
the logit form of a Tullock contest \citep{tullock1980efficient} whose numbers term generalises Lanchester's laws: under the square law with per-player strength proportional to $\bar g^{\beta}$, $\gamma_n/\gamma_g=2/\beta$; under the linear law, $1/\beta$.
The fight winner takes the objective; a lone committer takes it uncontested.
The second front is a Blotto-type contest \citep{roberson2006colonel}, won with probability $m_i^{\gamma_f}/(m_i^{\gamma_f}+m_j^{\gamma_f})$ and worth a tower.
$\gamma_n$ and $\gamma_g$ are estimated from fights in the core games (\Cref{sec:fights}), outside the choice model.

\paragraph*{Utilities and equilibrium}
Team $i$'s expected payoff splits into the objective part $U^{\mathrm{obj}}_i$ and the rest (fight and front) $U^{\mathrm{rest}}_i$, both in win-probability points.
Choices follow a heterogeneous logit quantal response equilibrium \citep{mckelvey1995quantal}:
\begin{equation}
\begin{split}
\log\sigma_i(a)={}&\lambda_i[\nu_i U^{\mathrm{obj}}_i(a,\sigma_j)+U^{\mathrm{rest}}_i(a,\sigma_j)]\\
&-\kappa_i\tau(a)-\psi_i T_k(a)-\omega_{k,i}k-\omega_{m,i}m+\mathrm{const},
\end{split}
\label{eq:qre}
\end{equation}
where $\lambda_i$ is decision precision ($\lambda\to\infty$ nests Nash), $\nu_i$ the weight on the objective's own value ($\nu=1$: the team maximises win probability), $\kappa_i$ the cost of early positioning, $T_k$ the travel minutes of the $k$ nearest players, and $\omega$ the per-player costs of committing to the pit or the second front.
Team parameters are shrunk to global values.

\begin{proposition}[Contest threshold]\label{prop:threshold}
Suppose team $i$ either contests, reaching win probability $W(\eta+G)$ with probability $p$ and $W(\eta-L)$ otherwise, or concedes and trades, reaching $W(\eta-b+c)$. With logit choice at precision $\lambda$,
\[
P(\text{contest}\mid\eta)=\Lambda\big(\lambda S(\eta)[p(\eta)-\pi^*(\eta)]\big),
\]
\[
\pi^*(\eta)=\frac{W(\eta-b+c)-W(\eta-L)}{W(\eta+G)-W(\eta-L)},
\]
with stake $S(\eta)=W(\eta+G)-W(\eta-L)$.
\end{proposition}
The proposition separates the channels that the decomposition below measures.
A trailing team has a lower fight probability $p$ (Lanchester), but where $W$ is convex its break-even $\pi^*$ falls (curvature), and a valuable second front raises $\pi^*$ (Blotto).

\begin{proposition}[Uniqueness]\label{prop:unique}
If $\nu_A=\nu_B=1$ and costs depend only on a team's own action, the game is zero-sum in win probability up to own-action terms, and the logit QRE is the unique saddle point of
\[
F(x,y)=x^\top Uy-\tfrac{c_A^\top x-H(x)}{\lambda_A}+\tfrac{c_B^\top y-H(y)}{\lambda_B}
\]
over the two simplices, with $H$ the Shannon entropy.
\end{proposition}
Proofs are in \Aref{app:proofs}.
Under the estimated $\nu\neq1$ the game is no longer zero-sum; we solve each state by damped fixed-point iteration from uniform strategies.

\paragraph*{Estimation and tests}
We estimate $\theta_i=(\lambda_i,\nu_i,\kappa_i,\psi_i,\omega_{k,i},\omega_{m,i})$ by nested pseudo-likelihood \citep{aguirregabiria2007sequential} and compute intervals from 300 series-block bootstrap replicates \citep{kunsch1989jackknife}.
Empirical content comes from fixing payoffs at outcome-only values \citep{haile2008empirical}; we compare the fit with a saturated multinomial benchmark.
Three named tests follow from \Cref{prop:threshold}: $H_{\mathrm{Simpson}}$ (matched on $\hat p$ and stake, trailing teams contest no less than leading teams), $H_{\mathrm{trade}}$ (among non-committers, second-front use rises as $\hat p$ falls) and $H_{\mathrm{winmax}}$ ($\nu=1$).

\paragraph*{Four-channel decomposition}
Let $G(S)$ be the equilibrium gap $P(k\geq3\mid\text{trailing})-P(k\geq3\mid\text{leading})$ at the observed states when only the channels in $S\subseteq\{L,C,F,B\}$ are on: $L$ (Lanchester; off: $\gamma_n=0$), $C$ (curvature; off: values taken at $W=0.5$), $F$ (second front; off: front values zero) and $B$ (team behaviour; off: pooled parameters).
We solve all 16 subsets and allocate $G(\{L,C,F,B\})-G(\varnothing)$ by Shapley values \citep{shapley1953value}.
Convexity-driven risk taking by trailing sides is known in isolation \citep{cabral2003rd,skinner2011scoring}; the decomposition puts it in one model with the fight law and the second front and measures each in the same unit.

%% file: sections/05_results.tex
\section{Results}
\label{sec:results}

Each subsection answers one question. Intervals are 95\% series-block bootstrap intervals (1,000 resamples unless stated; 500 for the context model); multiple contrasts are controlled by Benjamini--Hochberg \citep{benjamini1995controlling}.

\subsection{Trailing teams leave the pit empty}
\label{sec:empty}

\figwide{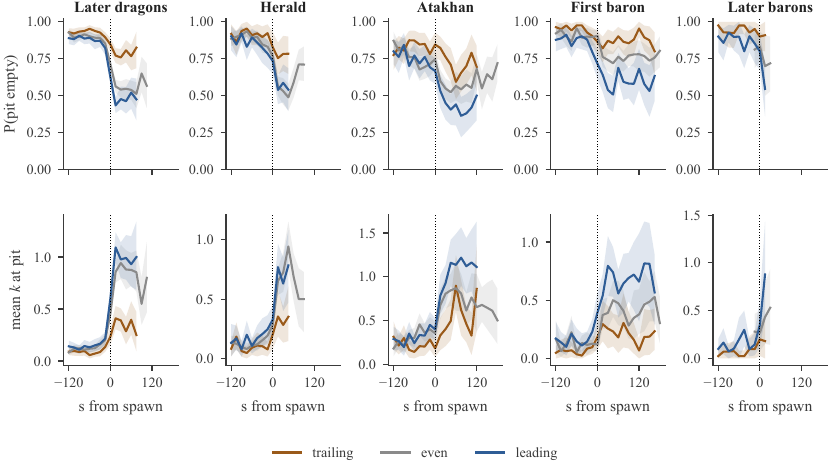}{Trailing teams stay away once the objective is up. Share of team-snapshots with nobody at the pit (top) and mean number of players at the pit (bottom), by gold state, from 120\,s before to 180\,s after the spawn (V layer, risk set, pit radius 1,500; bands: 95\% intervals; cells with fewer than 30 setups omitted).}{fig:curves}

Before the spawn the three gold states look alike: at every objective most teams have nobody at the pit, and the curves overlap (\Cref{fig:curves}).
They separate after the spawn.
Leading and even teams move players in, trailing teams much less so.
\Cref{fig:contrasts} maps the trailing-minus-leading contrast over objectives and time.
\V{vHeadlineCells} of \V{vContrastCells} cells at the primary radius are headlines (same sign at radii 1,000, 1,500 and 2,000; $q_{\mathrm{BH}}<0.05$), and \V{vHeadlineTrailLess} of them point the same way: more empty pits and fewer players for trailing teams (\Cref{tab:vlayer}).
At the first baron the share of empty pits is \V{vCBaronFortyfive}\pp{} \V{vCBaronFortyfiveCI} higher for trailing teams at $s{+}45$; at Atakhan \V{vCAtakhanFortyfive}\pp{} \V{vCAtakhanFortyfiveCI}, with \V{vCAtakhanFortyfiveMeanK} \V{vCAtakhanFortyfiveMeanKCI} fewer players; at respawned dragons \V{vCDragonTwoFifteen}\pp{} \V{vCDragonTwoFifteenCI} at $s{+}15$, where trailing teams also put three players at the pit \V{vCDragonTwoFifteenSelf}\pp{} \V{vCDragonTwoFifteenSelfCI} less often.
The only pre-spawn headline is at the herald, where trailing teams leave the pit empty \V{vCHeraldMinusSeventyfive}\pp{} \V{vCHeraldMinusSeventyfiveCI} more often at $s{-}75$.

\figwide{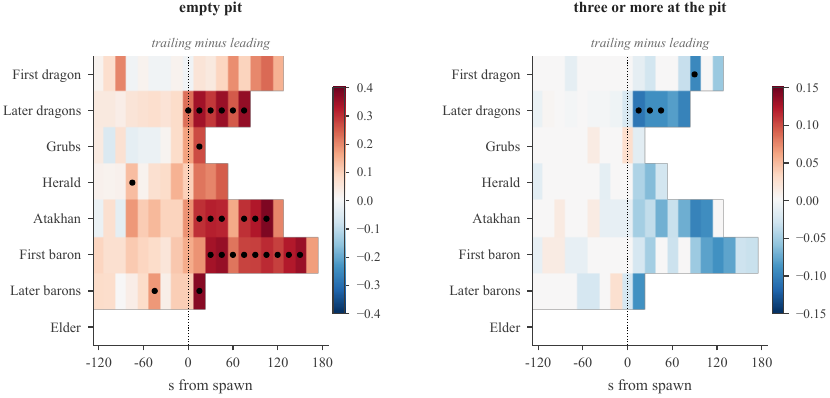}{The contrast is post-spawn and shared across objectives. Trailing-minus-leading difference in the share of empty pits (left) and of three or more players at the pit (right), by objective group and time (V layer, radius 1,500). Dots mark headline cells: same sign at all three pit radii and $q_{\mathrm{BH}}<0.05$.}{fig:contrasts}

The contrast does not come from tracking error.
Dropping every snapshot in which more than half of the positions are held or interpolated keeps \V{vHeadlineCleanKept} of the \V{vHeadlineCells} headline cells as headlines and the sign of \V{vHeadlineCleanSign} of the cells that remain estimable; the two sets of contrasts correlate at \V{cleanCorr} (Appendix \Cref{fig:cleanfull}).
On Riot's exact frames (R layer), the first baron at $s{+}60$ shows the same pattern: \V{rBaronEmpty}\pp{} \V{rBaronEmptyCI} more empty pits and \V{rBaronMeanK} \V{rBaronMeanKCI} players (\Cref{tab:rlayer}, \Cref{fig:rlayer}a).
Atakhan points the same way on the R layer (\V{rAtakhanEmptyLo}--\V{rAtakhanEmptyHi}\pp{} across radii) with $q=\V{rAtakhanQ}$; the 60-s frames are too coarse to reach the headline bar there.

\subsection{The estimator decides what is significant}
\label{sec:audit}

\figwide{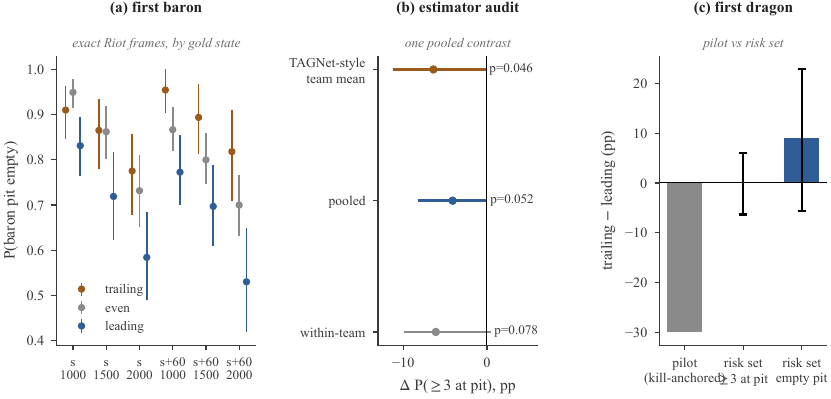}{Exact Riot frames confirm the baron contrast and expose two estimator artefacts. (a)~Share of empty baron pits by gold state at the spawn and 60\,s later, for three pit radii. (b)~One pooled contrast (share of teams with three or more at the pit at $s{+}60$, all first spawns) under three estimators; only the unweighted team-mean difference crosses $p<0.05$ (\V{auditTagnetP} against \V{auditPooledP}; both within bootstrap noise of the line). (c)~The first-dragon gradient of the kill-anchored pilot against the spawn-anchored risk-set contrasts.}{fig:rlayer}

The same data support different conclusions depending on the estimator (\Cref{fig:rlayer}b).
Pooling the first spawns, the trailing-minus-leading difference in committing three players at $s{+}60$ is \V{auditTagnet}\pp{} \V{auditTagnetCI} with $p=\V{auditTagnetP}$ as a difference of unweighted team means (team-level averaging in the spirit of TAGNet's paired team-level test).
Weighting by exposure gives \V{auditPooled}\pp{} ($p=\V{auditPooledP}$), and the within-team contrast over teams observed in both states \V{auditWithin}\pp{} ($p=\V{auditWithinP}$).
The point estimates agree; the $p$-value crosses 0.05 with the estimator, although \V{auditTagnetP} against \V{auditPooledP} lies within the bootstrap Monte-Carlo noise of about $\pm0.007$.

Anchoring matters more.
In our kill-anchored pilot on Worlds~2025, \V{pilotDragonTrail}\% of trailing and \V{pilotDragonLead}\% of leading teams that did not take the first dragon had contested it (only \V{pilotDragonLeadN} leading cases).
Anchored at the spawn, with the risk set, the first-dragon contrast at $s{+}60$ is \V{rDragonSixtyEmpty}\pp{} \V{rDragonSixtyEmptyCI} for empty pits, and committing three players is equally rare in both states (\Cref{fig:rlayer}c).
At the first dragon both teams are still laning: at least \V{rDragonEmptyAtSpawnMin}\% of pits are empty at the spawn in every gold state.
The pilot gradient was produced by selecting on who took the objective.

\subsection{Arriving first decides most river objectives}
\label{sec:race}

For each spawn we record which team first has three players within 1,500 units of the pit.
That team takes the objective within 180\,s in \V{raceDragonTwo} of respawned dragons ($n=\V{raceDragonTwoN}$), \V{raceHerald} of heralds and \V{raceBaron} of first barons, but in only \V{raceAtakhan} of Atakhans ($n=\V{raceAtakhanN}$; \Cref{fig:arrival}).
The median first arrival is \V{raceMedianT}\,s after the spawn: setups are decided after the clock, not before it.

\subsection{Six patterns, reproducible}
\label{sec:patterns}

No finite $K\leq8$ is selected by any criterion on the V layer: held-out log-likelihood per setup rises from \V{hvTestK1}\,(K=1) to \V{hvTestK6} at $K=6$ and \V{hvTestK8} at $K=8$, with steps of \V{hvStepFiveSix}, \V{hvStepSixSeven} and \V{hvStepSevenEight} nats from $K=5$ on (\Cref{fig:selection}, \Cref{tab:selection}).
The 7-snapshot grid and both alternative formation radii show the same rise; $K=6$ is the largest $K$ fitted there, and the criteria still improve at it (Appendix \Cref{tab:hasserobust}).
We therefore report $K=6$ as a convenient level of detail, not a selected one, and show $K=5$ and $K=8$ in the appendix; the rigid early-river patterns and the gathered late patterns described below appear at both.

\figwide{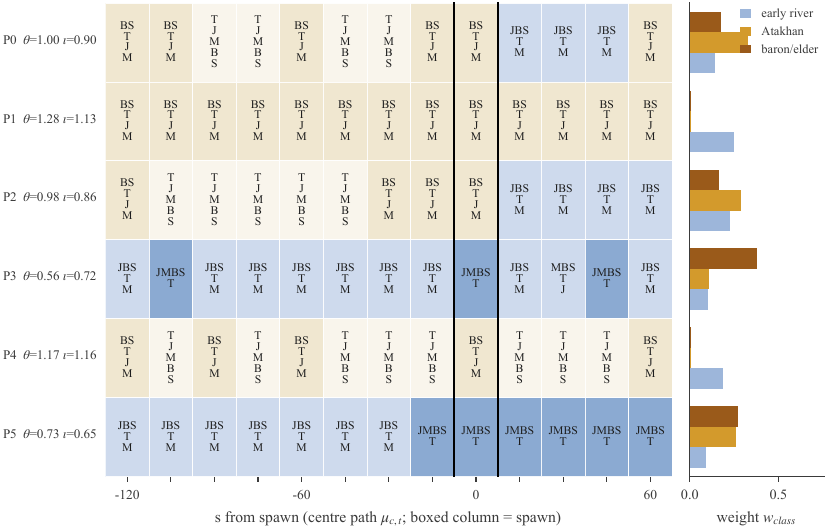}{The pattern dictionary ($K=6$). Each row is a pattern's centre path $\mu_{c,t}$ on $\Pi_5$, coloured by the number of blocks (darker = more gathered), with rigidity $\theta$ and inertia $\iota$ in Hasse steps; the boxed column is the spawn. Right: mixture weight of each pattern in the three objective classes.}{fig:patterns}

The six patterns separate first by objective class, then by how they gather (\Cref{fig:patterns}, \Cref{tab:patterns}).
Patterns~1 and~4 carry weight \V{patOneWEarly} and \V{patFourWEarly} at early river objectives and almost none elsewhere; they keep the bot lane together and everyone else apart (\texttt{\V{patOneAtSpawn}} at the spawn) and are the most rigid ($\theta=\V{patOneTheta}$, \V{patFourTheta}).
Patterns~3 and~5 gather four players with the top laner apart (\texttt{\V{patThreeAtSpawn}}) at the spawn and dominate late objectives (weights \V{patThreeWLate} and \V{patFiveWLate}); they are the least rigid ($\theta=\V{patThreeTheta}$, \V{patFiveTheta}) and the least inert ($\iota=\V{patThreeIota}$, \V{patFiveIota}): late setups change shape quickly.
Patterns~0 and~2 alternate between the spread and the bot-plus-support formations before the spawn and serve Atakhan most.
The model switches patterns with probability $\xi=\V{xiSix}$ per snapshot, so a setup typically keeps one pattern but can change it.

The `only gather' model $H_{\mathrm{mono}}$ (at $K=5$) fits \V{monoNats} nats worse on the full data with the same number of parameters, but its held-out deficit is only \V{monoDiff} nats per setup (fold SD \V{monoSD}; \V{monoFolds} of 5 folds favour the unrestricted model).
Splits, players peeling off to ward or flank, exist but are a small part of the structure.

\figwide{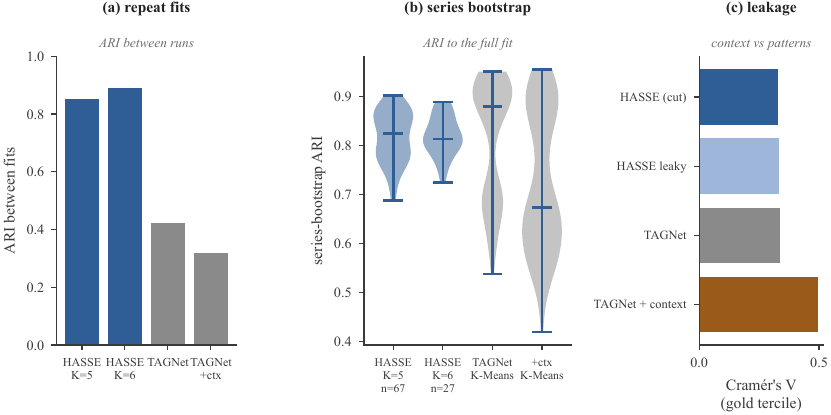}{Whole-model HASSE refits and TAGNet's K-Means re-runs are comparably stable, while single TAGNet runs agree less than best-of-ten HASSE fits. (a)~Agreement between repeated full-data fits: five HASSE fits, each the best of ten restarts, and ten single TAGNet runs (with and without its context vector). (b)~Series-bootstrap ARI: HASSE refits the whole model on each resample, TAGNet re-runs only its K-Means step on fixed embeddings. (c)~Cramér's V between the assignment and the gold tercile when context is kept out of the pattern layer and when it is fed in.}{fig:stability}

\paragraph*{Against TAGNet}
We ported TAGNet to the same V-layer tracks (it uses the \V{tgN} setups with at least 6 of 7 pre-spawn snapshots; three-layer distance-biased graph attention, fixed GRU, $K$-Means at $K=6$; ten seeds with and without its context vector).
Two TAGNet runs that differ only in the seed agree at ARI \citep{hubert1985comparing} \V{tgSeedAri} (\V{tgSeedAriCtx} with context); five independent HASSE fits, each the best of ten restarts, agree at \V{hsCrossKSix}, so this contrast compares best-of-ten fits with single runs (\Cref{fig:stability}a, \Cref{tab:hassetagnet}).
Refitting the whole HASSE model on series-bootstrap resamples gives ARI \V{hsBootKSix} (5\% quantile \V{hsBootKSixQ}, \V{hsBootKSixN} resamples), comparable to re-running only TAGNet's final $K$-Means step on fixed embeddings (\V{tgBootAri}).
TAGNet's clusters agree with its own counting baseline at ARI \V{tgCountAri}, and its silhouette at $K=6$ is \V{tgSilSix}.
Feeding the context vector in raises the association between TAGNet clusters and the gold tercile from Cramér's $V=\V{tgVgold}$ to \V{tgVgoldCtx} ($\times$\V{tgLeakRatio}); the same leak into HASSE's mixture weights moves it from \V{lkHasseCut} to \V{lkHasseLeaky}, because the weights fix only the first snapshot and the lattice emissions carry the rest.

\subsection{What context does to patterns}
\label{sec:context}

With the patterns fixed, a soft-label multinomial logit relates occupancy to context (\Cref{fig:context}).
Time and objective class carry most of the variation.
A one-SD gold lead raises the gathered late pattern~5 by \V{ameGoldFive} \V{ameGoldFiveCI} and lowers patterns~0 (\V{ameGoldZero}) and~4 (\V{ameGoldFour}); holding the dragon soul point raises pattern~4 by \V{ameSoulFour} \V{ameSoulFourCI}.
\V{ameSigCount} of \V{ameCells} effects have intervals that exclude zero.

\subsection{Fight laws}
\label{sec:fights}

\figcol{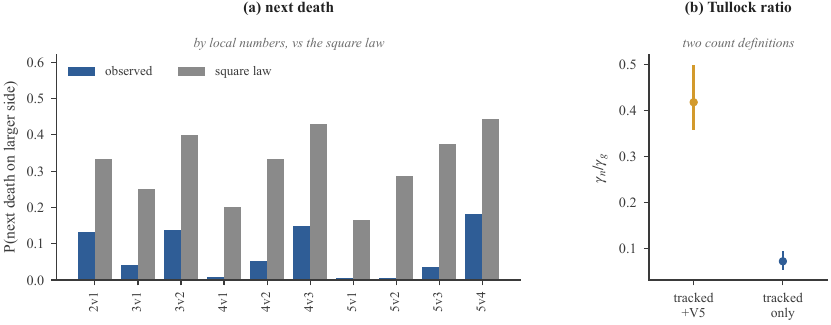}{The larger side rarely loses the next player. (a)~Probability that the next death falls on the larger side, by local numbers within 2,000 units, against the stochastic square law. (b)~The Tullock ratio $\gamma_n/\gamma_g$ under two count definitions.}{fig:fights}

The larger side almost never loses the next player (\Cref{fig:fights}a): \V{sSevenFiveThree} at 5v3 against \V{sSevenFiveThreeSq} under the square law, \V{sSevenTwoOne} at 2v1 against \V{sSevenTwoOneSq}.
The slope of the next-death logit on $\ln(n_{\mathrm{large}}/n_{\mathrm{small}})$ is \V{sSevenBn} \V{sSevenBnCI} with tracked counts, between the linear law (0) and the square law ($-1$), and \V{sSevenBnVfive} \V{sSevenBnVfiveCI} once killers and assisters from v5 are added; the levels sit far below the square law at every ratio.
In the Tullock fight law~\eqref{eq:tullock}, fitted on \V{tulGnN} fights with tracked counts, $\gamma_n=\V{tulGn}$ \V{tulGnCI} and $\gamma_g=\V{tulGg}$ \V{tulGgCI} (\Cref{tab:fights}).
With $\gamma_n/\gamma_g=\V{tulRatio}$, going from 4 to 5 players is worth as much as a \V{tulFiveFourGoldEq}\% per-player gold edge; side does not matter ($\gamma_{\mathrm{side}}=\V{tulSide}$ \V{tulSideCI}).
Adding v5 participants, who are counted because they took part, raises $\gamma_n$ to \V{tulGnVfive}; we use the tracked counts and treat this as a sensitivity.

\subsection{What the equilibrium reproduces}
\label{sec:qre}

\figwide{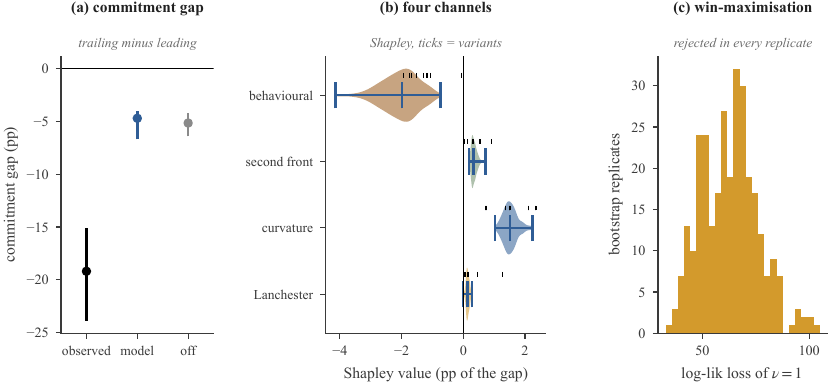}{The allocation game reproduces a quarter of the commitment gap, almost all of it from state variables rather than the four channels. (a)~Trailing-minus-leading gap in the share of teams with three or more players at the pit at the commitment time: observed, LB-QRE with all channels, and with all channels off. (b)~Shapley values of the four channels over 300 bootstrap replicates (ticks: specification variants). (c)~Log-likelihood loss from imposing $\nu=1$ in each replicate (a nested model, so the loss is non-negative by construction; its size is the result).}{fig:qre}
\input{tables/tab_qre}

LB-QRE is fitted on \V{qContests} contests (\Cref{tab:qre}).
Decision precision is $\lambda=\V{qLam}$ \V{qLamCI} per win-probability unit, so a one-point difference in expected win probability changes choice odds by \V{qLamOnePoint}\%.
The weight on the objective's own value is $\nu=\V{qNu}$ \V{qNuCI}: choices respond to the value of fights and fronts and to their costs, early positioning ($\kappa=\V{qKappa}$), travel ($\psi=\V{qPsi}$ per minute) and committed players ($\omega_k=\V{qOmegaK}$, $\omega_m=\V{qOmegaM}$), and barely to the objective itself.
Imposing $\nu=1$ ($H_{\mathrm{winmax}}$) costs \V{qWinmaxPoint} log-likelihood points on the data (likelihood ratio about \V{qWinmaxLR}); across bootstrap replicates the loss has mean \V{qWinmax} \V{qWinmaxCI} and is at least \V{qWinmaxMin} in all \V{qReps}.
Without action costs the fit collapses to $\lambda=0$ (log-likelihood \V{nocostLL} against \V{qLL}).
The saturated benchmark reaches \V{qLLsat}: the equilibrium restrictions cost likelihood, as they should if they have content.
Only \V{qConvShare}\% of the \V{qReps} bootstrap replicates (\V{qConvN}) meet the pseudo-likelihood convergence tolerance; mean $\nu$ is \V{qNuConv} in the converged and \V{qNuNonconv} in the other replicates, and the mean Shapley values differ by at most \V{qConvShapDiff}\pp.

The observed commitment gap is \V{gapObs}\pp{} \V{gapObsCI}; the model produces \V{gapModel}\pp{} \V{gapModelCI}, about \V{gapShare} of it (\Cref{fig:qre}a).
Essentially all of the model's gap is there with every channel off (\V{gapNone}\pp), so the four channels together add only \V{gapChannels}\pp{} net: the gold term of the fight law, the alive counts, and the travel and feasibility constraints still differ between gold states.
The channels then pull in opposite directions (\Cref{tab:shapley}): curvature of $W$ moves trailing teams towards contesting (\V{shC}\pp{} \V{shCCI}), team behaviour moves them away (\V{shB}\pp{} \V{shBCI}), and the Lanchester (\V{shL}) and second-front (\V{shF}) channels are small.
Curvature exceeds Lanchester by \V{shCminusL}\pp{} \V{shCminusLCI}.
Every sign holds across the six non-degenerate specification variants (pit radius 1,000 and 2,000, result values, dropping held or interpolated decisions, the V5 fight law, $\nu=1$), with Shapley magnitudes up to \V{varShapMax}\pp{}; the two no-cost variants collapse to $\lambda=0$, so their agreement is vacuous (Appendix \Cref{tab:qrevariants}).

The named tests sharpen the reading.
Unconditionally trailing teams contest less ($H_{\mathrm{Simpson}}$ raw gap \V{simpRaw} \V{simpRawCI}); matched on fight probability and stake the coefficient is \V{simpMatched} \V{simpMatchedCI}, in the predicted direction and not distinguishable from zero.
Among teams that do not commit, second-front use does not rise as fight probability falls ($H_{\mathrm{trade}}$ slope \V{trade} \V{tradeCI}).
In the model, trailing teams stay away from objectives because their fights are worse and their players are farther away, not because they prefer to trade.

\subsection{Team fingerprints are mostly noise}
\label{sec:fingerprints}

\figwide{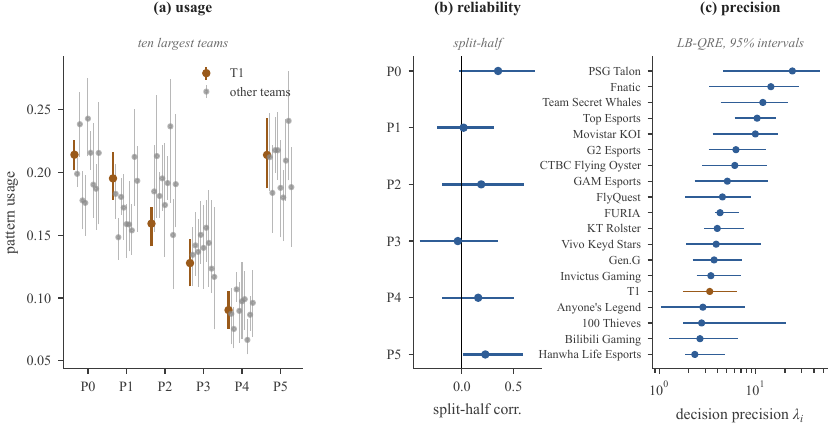}{Team fingerprints come with wide intervals. (a)~Pattern usage of the ten teams with most setups (T1 in red) with series-bootstrap intervals. (b)~Split-half reliability of team usage per pattern. (c)~Team decision precision $\lambda_i$ with 95\% intervals.}{fig:fingerprints}

TAGNet reports team fingerprints without intervals.
With intervals the picture is sober (\Cref{fig:fingerprints}).
Split-half correlations of team usage range from \V{shMin} to \V{shMax}, and the interval includes zero for \V{shZeroIn} of 6 patterns.
Team precision parameters overlap: T1's $\lambda=\V{lamTOne}$ \V{lamTOneCI} and Gen.G's \V{lamGenG} \V{lamGenGCI} are not distinguishable.
At \V{setupsTeamMin} to \V{setupsTeamMax} setups per team, patterns describe the league; they do not yet separate its teams.

%% file: tables/tab_qre.tex
\begin{table}[t]
\centering
\small
\caption{\textbf{Choices respond to fights, travel and player costs, and little to the objective's own value.} LB-QRE global parameters (1,454 contests, 19 teams with own parameters; 300 series-block bootstrap replicates).}
\label{tab:qre}
\setlength{\tabcolsep}{4pt}
\begin{adjustbox}{max width=\linewidth}
\begin{tabular}{llrr}
\toprule
 & Meaning & Est. & 95\% CI \\
\midrule
$\lambda$ & decision precision (1/win-prob.) & 4.297 & [3.005, 6.457] \\
$\nu$ & weight on the objective's own value & 0.018 & [0.013, 0.026] \\
$\kappa$ & cost of early positioning & 1.794 & [1.674, 1.927] \\
$\psi$ & cost per travel minute & 1.813 & [1.476, 2.170] \\
$\omega_k$ & cost per player at the pit & 0.185 & [0.126, 0.239] \\
$\omega_m$ & cost per player on 2nd front & 0.688 & [0.639, 0.772] \\
\midrule
$\ell$ & log-lik. LB-QRE / saturated & -8,343 & -6,826 \\
$H_{winmax}$ & log-lik. loss of $\nu=1$ & 62.9 & [40.1, 92.8] \\
\bottomrule
\end{tabular}
\end{adjustbox}
\end{table}

%% file: sections/06_related.tex
\section{Related Work}
\label{sec:related}

\paragraph*{Tactics in esports}
TAGNet \citep{zhang2026tagnet} is the closest study and the template we mirror section by section; we replace its learned embedding by objects with units, its fixed $K$ by held-out evidence, and its context-in-the-representation design by a cut.
Spatio-temporal team behaviour in MOBAs was first studied in \emph{Dota~2}, where skill levels differ in how teams spread across zones \citep{drachen2014skill}.
In \emph{Counter-Strike}, actions have been valued by changes in win probability \citep{xenopoulos2020valuing} and trajectory data released at scale \citep{xenopoulos2022esta}.
LoL research has focused on outcome prediction \citep{hodge2021win,yang2022predicting}, player rating \citep{maymin2020smart,debois2025pandaskill}, team networks \citep{mora2019team} and match-level performance indicators \citep{novak2020performance}.
We use win probability the way action-valuation work does \citep{decroos2019actions,xenopoulos2020valuing}, but to price the outcomes of a game between teams rather than to score individual actions.

\paragraph*{Space and formations in team sports}
Dominant regions \citep{taki2000visualization} are Voronoi cells of player positions; we compute them on a navigation graph so that walls count.
Role-based alignment of formations \citep{lucey2013representing,bialkowski2014large} and the broader spatio-temporal literature \citep{gudmundsson2017spatio} treat formations as geometric templates.
We treat a formation as a set partition and use the lattice metric \citep{monjardet1981metrics} to count merges and splits.
Centred partition processes \citep{paganin2021centered} place priors around a centre partition; HASSE lets the centre move along a path over time and adds rigidity and inertia, inside a sticky HMM \citep{fox2011sticky}.
Sparse finite mixtures \citep{fruhwirth2019here} motivate letting the evidence, not a fixed $K$, set the number of patterns.

\paragraph*{Game theory in sport}
Penalty kicks and serves are the canonical tests of equilibrium play by professionals \citep{walker2001minimax,chiappori2002testing,palacios2003professionals}.
Risk taking by trailing contestants is documented in tournaments \citep{genakos2012interim} and derived from the convexity of win probability \citep{cabral2003rd,skinner2011scoring}, the logic behind pulling the goalie \citep{beaudoin2010pulling}.
Our explanation layer combines a quantal response equilibrium \citep{mckelvey1995quantal}, whose empirical content requires fixed payoffs \citep{haile2008empirical}, with Lanchester--Tullock fights \citep{lanchester1916aircraft,tullock1980efficient} and a Blotto second front \citep{roberson2006colonel}, estimated as in empirical industrial organisation \citep{aguirregabiria2007sequential}.
The Shapley decomposition \citep{shapley1953value} lets the convexity argument compete with the fight law and the second front in one unit.

%% file: sections/07_limitations.tex
\section{Limitations}
\label{sec:limitations}

Positions come from the broadcast minimap; with four or more other champions nearby only \V{trkCrowdFourLtSix} of positions are within 600 units, and formations agree exactly with Riot's in \V{formExact} of snapshots, which is why every headline rests on sign agreement across radii, an R-layer check or a sensitivity run without interpolated positions.
Vision, cooldowns, wave states and voice communication are unobserved, so every behavioural statement is conditional on the observed state.
The data cover one rule set and two international events; teams with few games carry wide intervals.
LB-QRE reproduces about a quarter of the observed commitment gap, almost all of it through state variables rather than its strategic channels; the rest is behaviour the allocation game in its present form does not capture.
Win probability and event values are outcome-only associations from the archive, not causal effects.

%% file: sections/08_conclusion.tex
\section{Conclusion}
\label{sec:conclusion}

A setup before an objective is an allocation of five players on a walled map against a public clock.
Measured that way, professional teams behave consistently: trailing teams leave the pit empty after the spawn, setups follow a small number of reproducible walks on the partition lattice, and commitment responds to fight odds, travel and the cost of each player far more than to the prize itself.
None of this requires a neural embedding, and much of what an embedding pipeline would report changes with the seed or the estimator.

The same objects apply wherever teams split players between a contested point and the rest of the map: other MOBAs, tactical shooters and team sports with set pieces.
The tracks, setups and code are released with the paper.

%% file: sections/A_appendix.tex
\section{Proofs}
\label{app:proofs}

\begin{proof}[Proof of \Cref{lem:metric}]
$\Pi_5$ ordered by refinement is a geometric, hence upper semimodular, lattice with rank $r(\pi)=5-|\pi|$.
For upper semimodular lattices $d(\sigma,\tau)=2r(\sigma\vee\tau)-r(\sigma)-r(\tau)$ is a metric equal to the shortest-path distance in the covering graph \citep{monjardet1981metrics}; substituting the rank gives the stated form.
Each covering step is one merge (rank $+1$) or one split (rank $-1$). On a path from $\sigma$ to $\tau$ with $M$ merges and $S$ splits, $M-S=r(\tau)-r(\sigma)$; on a shortest path also $M+S=d(\sigma,\tau)=2r(\sigma\vee\tau)-r(\sigma)-r(\tau)$. Hence $M=r(\sigma\vee\tau)-r(\sigma)=|\sigma|-|\sigma\vee\tau|$ and $S=|\tau|-|\sigma\vee\tau|$.
We verified all three statements exhaustively: breadth-first search on the Hasse diagram of the \V{latticeN} partitions (\V{latticeEdges} covering edges) reproduces $d$ for all pairs, and the triangle inequality holds for all triples.
\end{proof}

\begin{proof}[Proof of \Cref{prop:threshold}]
Contesting yields expected win probability $pW(\eta+G)+(1-p)W(\eta-L)$, conceding yields $W(\eta-b+c)$.
Their difference is $W(\eta-L)+pS-W(\eta-b+c)=S(p-\pi^*)$ with $S$ and $\pi^*$ as stated.
Under logit choice with precision $\lambda$ the probability of the first option is $\Lambda(\lambda\cdot\text{difference})$.
\end{proof}

\begin{proof}[Proof of \Cref{prop:unique}]
With $\nu=1$, $W_B=1-W_A$ implies that every outcome changes $B$'s win probability by minus $A$'s change, so $U_B=-U_A^\top$ in win-probability units, and the remaining terms ($\kappa\tau$, $\psi T_k$, $\omega$) depend only on a team's own action; write them $c_A^\top x$ and $c_B^\top y$ for mixed strategies $x,y$.
Team $A$'s logit best response to $y$ maximises $\lambda_A x^\top Uy-c_A^\top x+H(x)$, team $B$'s maximises $-\lambda_B x^\top Uy-c_B^\top y+H(y)$.
Dividing by $\lambda_A$ and $\lambda_B$, both are the first-order conditions of the saddle problem $\max_x\min_y F(x,y)$ with $F$ as stated.
$F$ is continuous, strictly concave in $x$ and strictly convex in $y$ on compact convex sets, so a saddle point exists (minimax theorem) and is unique by strictness; every logit QRE is such a saddle point and conversely.
\end{proof}

\section{HASSE: Estimation}
\label{app:hasse}

\paragraph*{Likelihood} For setup $e$ with pattern sequence $z_{1:T}$, the complete-data likelihood is $w_{\mathrm{cls}(e),z_1}\prod_t P(z_t\mid z_{t-1})\prod_t P_{z_t}(\pi_t\mid\pi_{t-1})\prod_{B\in\pi_t}P_{z_t}(g_B\mid\pi_t)$; snapshots after the kill are missing and drop out of the product.
\paragraph*{E step} Exact forward--backward over $K$ states.
\paragraph*{M step} (i)~For each pattern and snapshot, the centre $\mu_{c,t}$ is chosen among the 52 partitions by an exact Viterbi pass that adds a smoothness penalty of one unit per Hasse step between consecutive centres. (ii)~$(\theta_c,\iota_c)$ maximise a concave exponential-family likelihood whose gradient is observed minus expected distance, bounded in $[0,20]$. (iii)~Zone parameters by penalised multinomial logit (Gaussian penalty 0.1). (iv)~$w$, $\xi$, $q$ in closed form.
\paragraph*{Initialisation and selection} Ten restarts per $K$, even seeds from $k$-medoids on the lattice path distance; held-out log-likelihood on five series-grouped folds; BIC and ICL on the full data. Recovery on synthetic data drawn from the model (three planted patterns, \V{synthN} setups, three snapshots, five restarts): all restarts reach the same log-likelihood (spread \V{synthLLspread}), \V{synthPathsOK} centre paths are recovered exactly, rigidity and inertia lie within \V{synthThetaErr} and \V{synthIotaErr} of the planted values, and the first-snapshot assignment has ARI \V{synthAri} with the truth (\V{synthAriMissing} when the last snapshot is missing in 40\% of setups).
\paragraph*{$H_{\mathrm{mono}}$} The Viterbi step for $\mu$ is restricted to transitions $\mu_{t}\leq\mu_{t+1}$ in the refinement order.

\section{LB-QRE: Specification}
\label{app:qre}

\paragraph*{Commitment time and actions} $t_c$ is the first at-risk snapshot in $[s,s{+}180]$ with three or more players of either team within $R$ of the pit, otherwise the last at-risk snapshot (objective taken next, or window end). $m$ counts alive players in zone `far' at $t_c$; $\tau=1$ if the team already had at least $k$ players at the pit at the spawn. Players farther than 60\,s of movement at the spawn are infeasible for the pit.
\paragraph*{State values} For each contest, values for the objective (to either team), a fight won or lost by two kills, and one tower are read from the event regressions at the out-of-fold $W$ of the last v5 frame at least 30\,s before the spawn; dragons resolve to the gainer's next drake index.
\paragraph*{Estimation} Nested pseudo-likelihood alternates a conditional-logit maximisation in $\theta$ (analytic gradient) with a strategy update; team deviations are shrunk towards global values with prior standard deviations 1.0 for $\log\lambda$, $\kappa$, $\psi$ and 0.5 for $\log\nu$, $\omega_k$, $\omega_m$.
Teams with too few contests share a pooled `other' parameter.

\section{Additional Results}
\label{app:results}

\input{tables/tab_values}
\input{tables/tab_tracker}
\input{tables/tab_vlayer}
\input{tables/tab_rlayer}
\figwide{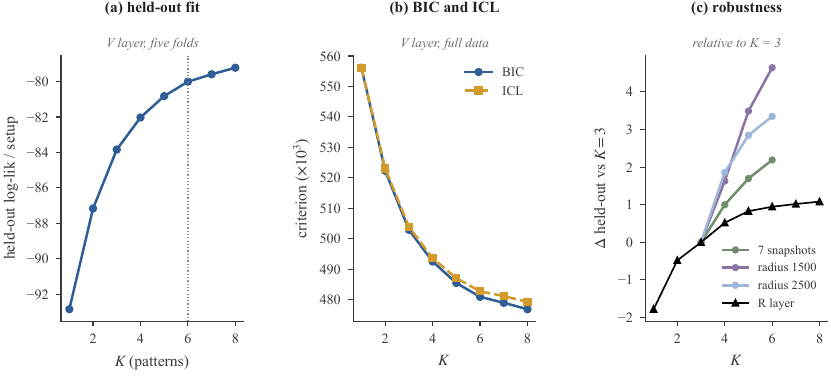}{The evidence favours rich structure with diminishing returns. (a)~Held-out log-likelihood per setup over five series-grouped folds against the number of patterns $K$ (V layer). (b)~BIC and ICL. (c)~Held-out log-likelihood relative to $K=3$ for the 7-snapshot grid, formation radii 1,500 and 2,500, and the R layer.}{fig:selection}
\input{tables/tab_selection}
\input{tables/tab_patterns}
\input{tables/tab_hasse_tagnet}
\figwide{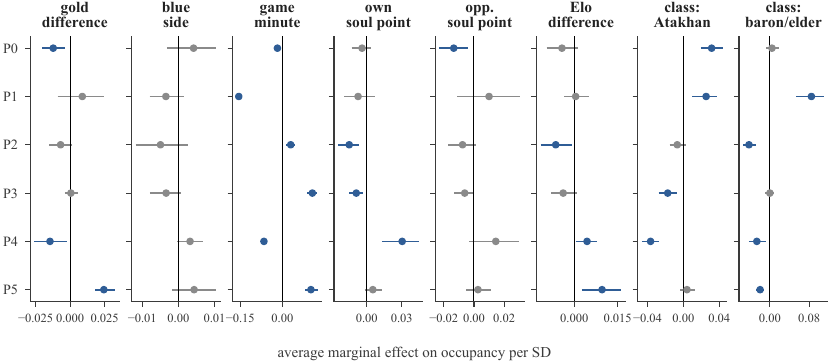}{Context shifts pattern use modestly once time and objective class are held fixed. Average marginal effect on each pattern's occupancy per standard deviation of each covariate, from a soft-label multinomial logit fitted after the patterns (cut design); blue: 95\% interval excludes zero.}{fig:context}
\input{tables/tab_fights}
\input{tables/tab_shapley}

\figcol{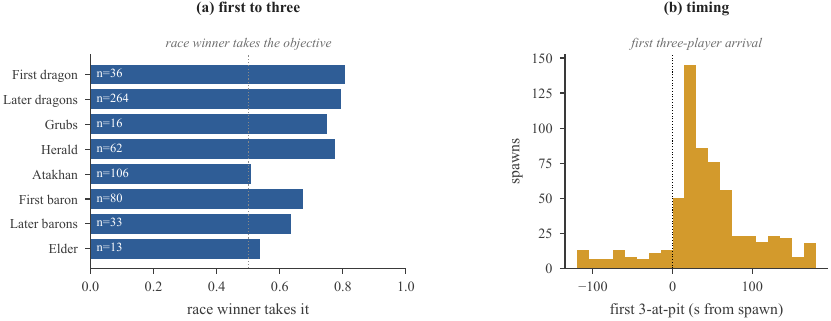}{The first team to reach three players at the pit usually takes river objectives, but not Atakhan. (a)~Share of spawns with a decided race in which the race winner took the objective within 180\,s. (b)~Time of the first three-player arrival relative to the spawn.}{fig:arrival}
\figwide{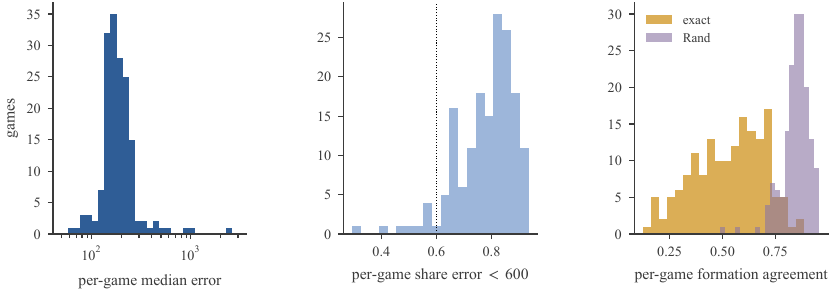}{Most games track well; a few are outliers. Tracker quality per game. Per-game median error, share of errors below 600 units, and formation agreement with v5 at held-out minutes.}{fig:trackerpergame}
\figwide{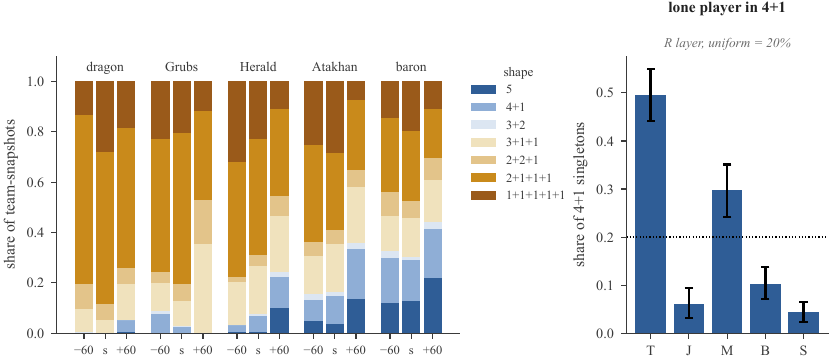}{The lone player in a 4+1 formation is usually the top laner. Formation shapes on the R layer at $s{-}60$, $s$ and $s{+}60$ by objective (left) and the role of the lone player in 4+1 formations against the uniform 20\% (right). The top laner is the lone player in \V{singleTop}\% \V{singleTopCI} of 4+1 formations, the mid laner in \V{singleMid}\% \V{singleMidCI}.}{fig:shapes}
\figcol{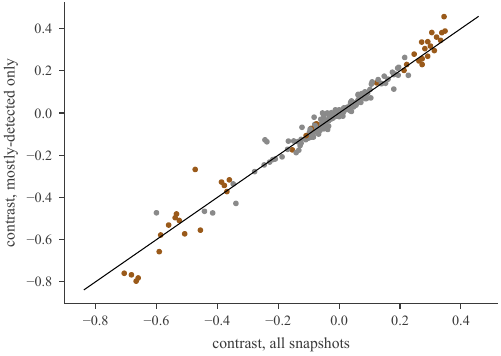}{Dropping mostly-interpolated snapshots leaves the contrasts essentially unchanged. V-layer trailing-minus-leading contrasts computed on all snapshots and on snapshots with at most half of the positions held or interpolated (radius 1,500; red: headline cells).}{fig:cleanfull}
\input{tables/tab_vocab}
\input{tables/tab_hasse_robust}
\figwide{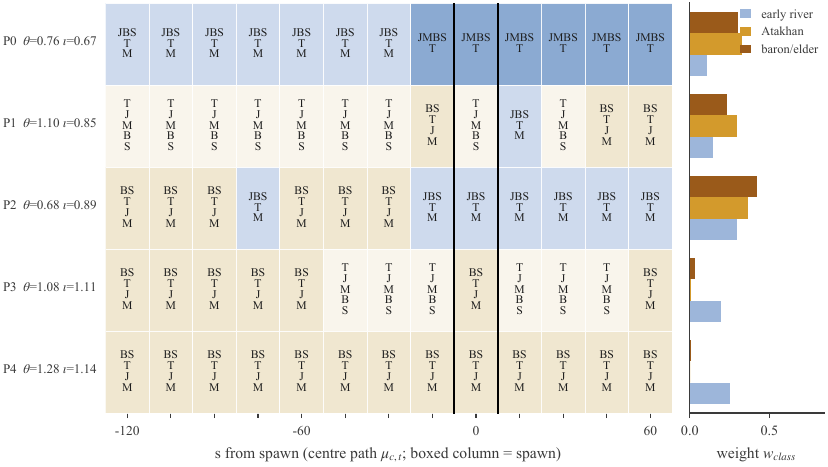}{The $K=5$ dictionary keeps the early-river and late patterns. Pattern dictionary at $K=5$ ($\xi=\V{xiFive}$).}{fig:patternsK5}
\figwide{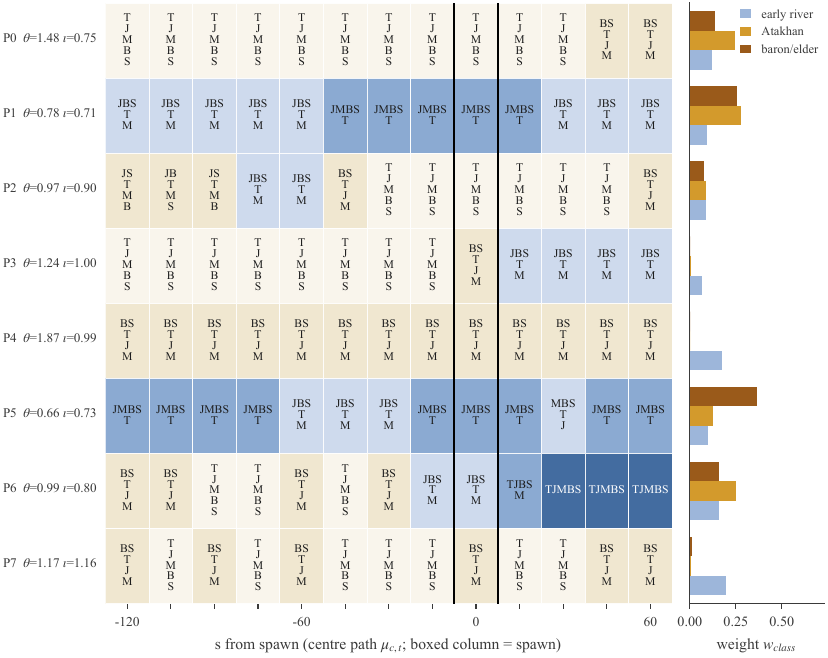}{The $K=8$ dictionary keeps the early-river and late patterns. Pattern dictionary at $K=8$ ($\xi=\V{xiEight}$).}{fig:patternsK8}
\figwide{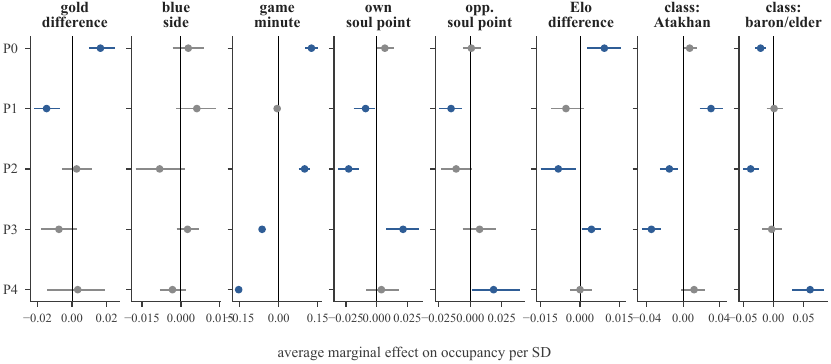}{Context effects are reported at $K=5$ as a robustness check. Context effects at $K=5$.}{fig:contextK5}
\figwide{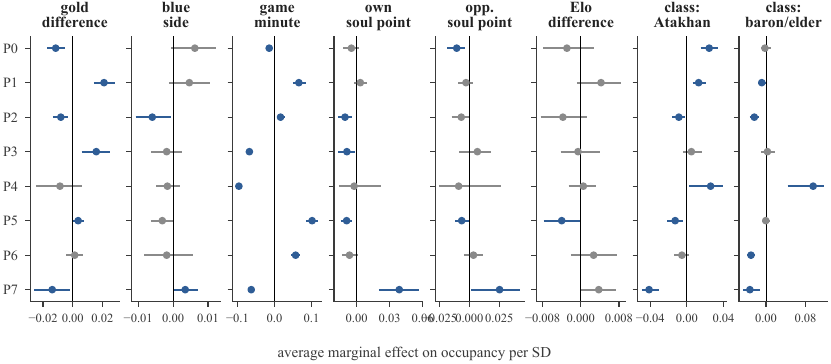}{Context effects are reported at $K=8$ as a robustness check. Context effects at $K=8$.}{fig:contextK8}
\input{tables/tab_context}
\figcol{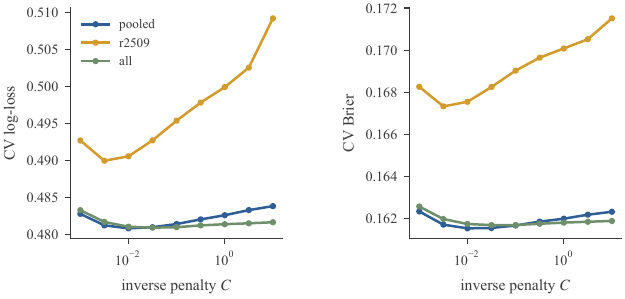}{Cross-validated loss is flat near the chosen penalty. Cross-validated log-loss and Brier score of $W$ against the inverse penalty $C$ for the three specifications.}{fig:wcv}
\input{tables/tab_wfit}
\figwide{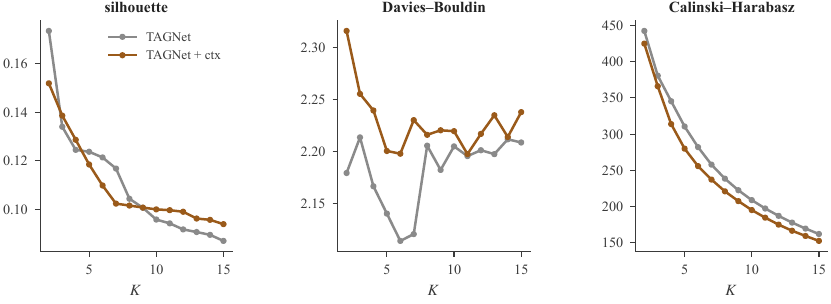}{TAGNet's own criteria fall as $K$ grows. TAGNet's own cluster criteria on our setups, $K=2..8$, with and without the context vector.}{fig:tagnetk}
\figwide{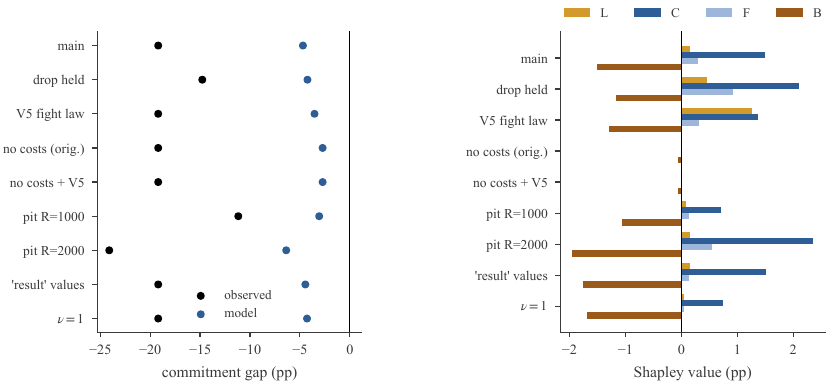}{Gaps and Shapley signs hold across the six non-degenerate specification variants. LB-QRE specification variants: observed and model commitment gaps (left) and the four Shapley values (right).}{fig:qrevariants}
\input{tables/tab_qre_variants}
\figwide{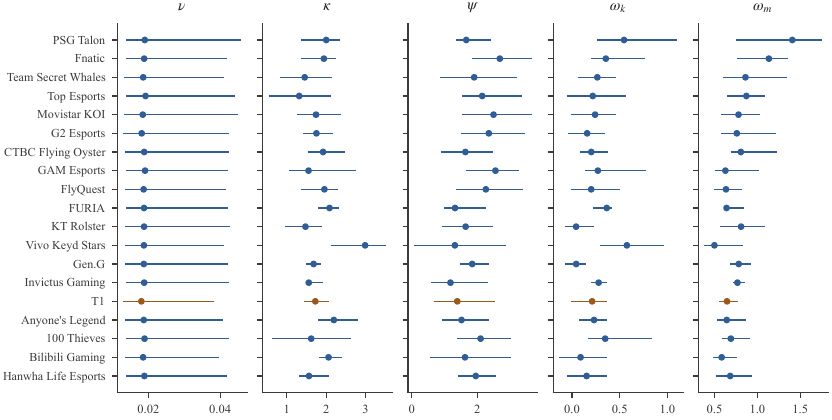}{Team parameters overlap widely. Team cost and valuation parameters with 95\% intervals (T1 in red), teams ordered by $\lambda$.}{fig:teamparams}
\input{tables/tab_team_params}
\input{tables/tab_usage}

\section{T1 and Faker}
\label{app:t1}

T1 contributes \V{tOneSetups} setups, the most of any team, and won Worlds~2025.
Every analysis here was specified before it was run, compares T1 or Faker with the rest of the league on the same measures, and reports intervals; null results are reported as such.

\paragraph*{A1 Pattern fingerprint}
T1's usage over the six patterns is \V{tOneUsage} (\Cref{tab:usage}).
Its nearest team is \V{tOneNearest} (total-variation distance \V{tOneNearestTV}; nearest in \V{tOneNearestP}\% of resamples), then \V{tOneSecond} and \V{tOneThird}.
Given the low split-half reliability of all fingerprints (\Cref{sec:fingerprints}), this places T1 inside the league's main cluster rather than apart from it.

\paragraph*{A2 Decision parameters}
T1's precision is $\lambda=\V{lamTOne}$ \V{lamTOneCI} and its precision on the objective $\lambda\nu=\V{lamObjTOne}$ \V{lamObjTOneCI}; neither separates T1 from Gen.G ($\lambda=\V{lamGenG}$ \V{lamGenGCI}) or from the league (\Cref{fig:teamparams}, \Cref{tab:teamparams}).

\figwide{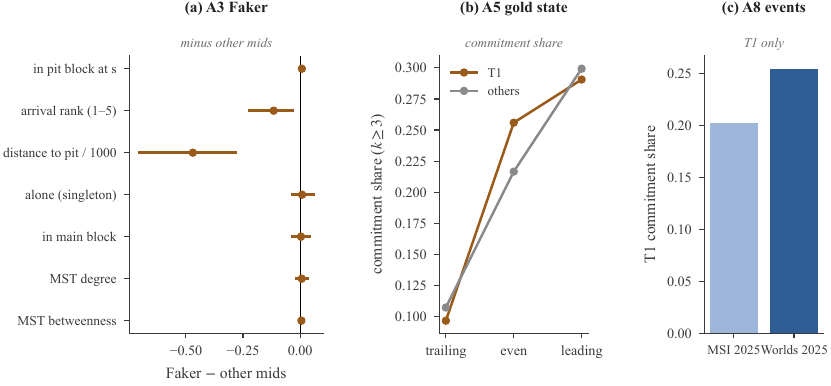}{Faker arrives earlier, but T1 behaves like the rest of the league. T1 and Faker in numbers. (a)~A3: Faker minus other starting mid laners on seven position measures in pre-spawn snapshots (distance in thousands of units). (b)~A5: commitment share by gold state, T1 against all other teams. (c)~A8: T1's commitment share at MSI and Worlds~2025.}{fig:t1}
\input{tables/tab_faker}
\input{tables/tab_t1_state}

\paragraph*{A3 Faker in the setup}
Against other starting mid laners ($n=\V{fakerN}$ vs \V{fakerOthersN} setups), Faker is closer to the pit before the spawn by \V{fakerDpit} \V{fakerDpitCI} units and arrives earlier in his team's order (rank \V{fakerRank} \V{fakerRankCI}); he is no more often alone, in the main block or central in the team's spanning tree (\Cref{tab:faker}, \Cref{fig:t1}a).
`Faker gets to the objective first' is measurable and true; `Faker holds the formation together' is measurable and not supported.

\paragraph*{A5 Behind and ahead}
T1's trailing-minus-leading commitment contrast is \V{tOneTL} \V{tOneTLCI}, the rest of the league's \V{othersTL}; the difference is \V{tOneMinusOthers} \V{tOneMinusOthersCI} (\Cref{tab:t1state}).
`T1 does not panic when behind' describes the whole league as well as T1.

\paragraph*{A6 Fights}
In \V{tOneFights} fights involving T1, the mean residual of the fight law is \V{tOneFightRes} \V{tOneFightResCI}: at equal numbers and gold T1 wins its fights as often as the law predicts.

\paragraph*{A7 The Worlds 2025 final}
\Cref{tab:ledger} and \Cref{fig:ledger} list all \V{ledgerSpawns} spawns of T1's 3--2 final against KT Rolster with both actions, the fight probability, the break-even of \Cref{prop:threshold} and the objective's value.

\figwide{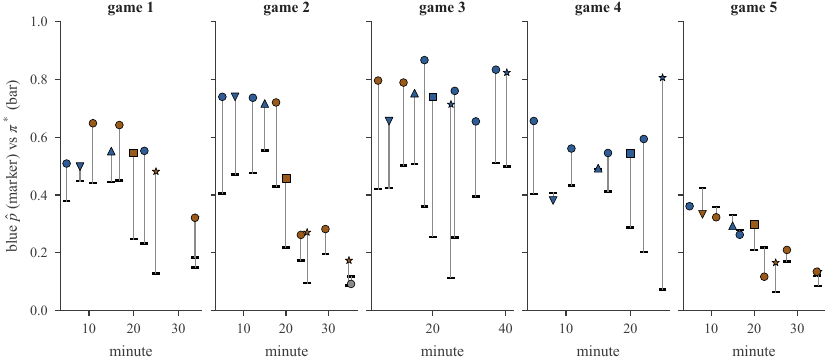}{The ledger lists every spawn of the Worlds~2025 final. A7: every spawn of the Worlds~2025 final. Markers: blue's fight probability $\hat p$ (colour = taker: blue/red/none; shape = objective); bars end at the break-even $\pi^*$.}{fig:ledger}
\input{tables/tab_ledger}

\paragraph*{A8 MSI and Worlds}
Under the same rules T1 committed three or more players in \V{tOneMSI} of contests at MSI and \V{tOneWorlds} at Worlds (\Cref{fig:t1}c).

\paragraph*{A9 Earlier rule sets}
Worlds 2023 and 2024 were played under different objective rules, so T1's earlier setups are outside this study.

\paragraph*{A10 Narratives and measures}
Three common narratives map onto measures here: `T1 does not panic when behind' (A5: no different from the league), `Faker controls the map' (A3: earlier arrival, no extra centrality) and `T1 wins fights at equal numbers' (A6: no residual).
Champion-pool depth, draft play and communication are outside what positions can measure.

\section{Dataset}
\label{app:dataset}

The release contains, for the \V{nGames} core games: the clock alignment per video, minimap registration parameters, 2-fps tracks with a source label per position (detected, interpolated, held, pinned, dead), the \V{nSpawns} spawn records, the \V{nSetups} team-setups on the 15-s grid with partitions, zones, $k$, $\tau^{(k)}$ and pit control and the navigation graph; the code that produces every figure and table is released with them.

%% file: tables/tab_values.tex
\begin{table}[t]
\centering
\small
\caption{\textbf{Every event is priced in win-probability points.} Event values in win-probability points at $W=0.5$, from outcome-only regressions on the archive ($\mathrm{logit}\,W_{t+180} = a\,\mathrm{logit}\,W_t + \beta(t)\,\mathrm{dir}$); `total effect' replaces $W_{t+180}$ by the result.}
\label{tab:values}
\setlength{\tabcolsep}{4pt}
\begin{adjustbox}{max width=\linewidth}
\begin{tabular}{lrrrr}
\toprule
Event & $n$ & $a$ & WPA (180 s) & Total effect \\
\midrule
Void grubs & 518 & 1.00 & 4.0 & 4.5 \\
Rift Herald & 489 & 1.00 & 6.2 & 4.9 \\
Dragon 1 & 4,612 & 1.03 & 6.0 & 8.1 \\
Dragon 2 & 3,544 & 1.01 & 8.5 & 6.8 \\
Dragon 3 (soul point) & 2,395 & 0.85 & 21.6 & 20.3 \\
Dragon 4 (soul) & 1,145 & 0.53 & 36.4 & 25.2 \\
Atakhan & 487 & 0.95 & 20.5 & 16.7 \\
Baron & 523 & 0.71 & 32.6 & 28.5 \\
Elder & 257 & 0.74 & 32.2 & 41.1 \\
Fight, per kill (away from objectives) & 14,563 & 0.84 & 8.7 & 8.1 \\
Fight, per kill (near objective) & 18,891 & 0.88 & 8.4 & 8.7 \\
Tower & 31,090 & 0.81 & 12.7 & 4.9 \\
\bottomrule
\end{tabular}
\end{adjustbox}
\end{table}

%% file: tables/tab_tracker.tex
\begin{table}[t]
\centering
\small
\caption{\textbf{Median tracking error is small against the map width and grows with crowding.} Tracker error against Riot V5 positions at held-out odd minutes (23,286 champion-minutes, 164 games; map units, map width 14,820). Crowding = other champions within 2,000 units.}
\label{tab:tracker}
\setlength{\tabcolsep}{4pt}
\begin{adjustbox}{max width=\linewidth}
\begin{tabular}{lrrrrrr}
\toprule
Subset & \% & Median & P90 & $<$300 & $<$600 & $<$1000 \\
\midrule
pooled & 100 & 177 & 3,739 & 0.70 & 0.79 & 0.82 \\
\midrule
detected & 78 & 146 & 589 & 0.83 & 0.90 & 0.91 \\
interpolated & 21 & 921 & 5,894 & 0.25 & 0.40 & 0.52 \\
held & 1 & 3,012 & 7,151 & 0.06 & 0.10 & 0.15 \\
\midrule
crowding 0 & 29 & 146 & 502 & 0.82 & 0.90 & 0.92 \\
crowding 1 & 28 & 148 & 1,374 & 0.81 & 0.88 & 0.89 \\
crowding 2 & 13 & 180 & 3,269 & 0.71 & 0.81 & 0.83 \\
crowding 3 & 11 & 214 & 6,061 & 0.61 & 0.72 & 0.76 \\
crowding 4+ & 19 & 621 & 7,675 & 0.39 & 0.49 & 0.58 \\
\bottomrule
\end{tabular}
\end{adjustbox}
\end{table}

%% file: tables/tab_vlayer.tex
\begin{table}[t]
\centering
\small
\caption{\textbf{Headline contrasts concentrate after the spawn.} V layer: headline trailing-minus-leading contrasts (sign agrees at three pit radii, $q_{BH}<0.05$). Cells = number of headline time points; the largest is shown. Units: pp for shares, players for mean $k$.}
\label{tab:vlayer}
\setlength{\tabcolsep}{4pt}
\begin{adjustbox}{max width=\linewidth}
\begin{tabular}{llrrrr}
\toprule
Objective & Quantity & Cells & Peak at & T$-$L & 95\% CI \\
\midrule
Herald & pit empty & 1 & $s-75$ & +12.3 & [3.4, 21.0] \\
Herald & mean $k$ & 1 & $s-75$ & $-$0.15 & [$-$0.27, $-$0.05] \\
Grubs & pit empty & 1 & $s+15$ & +26.2 & [6.9, 46.7] \\
Grubs & mean $k$ & 1 & $s+15$ & $-$0.36 & [$-$0.64, $-$0.12] \\
First dragon & $\geq$3 at pit & 1 & $s+90$ & $-$9.6 & [$-$17.0, $-$2.2] \\
First dragon & mean $k$ & 1 & $s+105$ & $-$0.39 & [$-$0.66, $-$0.11] \\
Later dragons & pit empty & 6 & $s+75$ & +35.0 & [19.0, 51.5] \\
Later dragons & $\geq$3 at pit & 3 & $s+15$ & $-$10.9 & [$-$15.7, $-$5.5] \\
Later dragons & mean $k$ & 7 & $s+75$ & $-$0.75 & [$-$1.05, $-$0.42] \\
Atakhan & pit empty & 6 & $s+105$ & +37.2 & [20.0, 54.0] \\
Atakhan & mean $k$ & 5 & $s+105$ & $-$0.84 & [$-$1.28, $-$0.43] \\
First baron & pit empty & 9 & $s+45$ & +34.8 & [18.7, 50.0] \\
First baron & mean $k$ & 8 & $s+135$ & $-$0.63 & [$-$1.03, $-$0.22] \\
Later barons & pit empty & 2 & $s+15$ & +36.4 & [15.4, 59.5] \\
Later barons & mean $k$ & 1 & $s-45$ & $-$0.27 & [$-$0.45, $-$0.12] \\
\bottomrule
\end{tabular}
\end{adjustbox}
\end{table}

%% file: tables/tab_rlayer.tex
\begin{table}[t]
\centering
\small
\caption{\textbf{Exact Riot frames confirm the baron contrast and the estimator dependence.} R layer (exact Riot frames) at $s{+}60$, risk set, pit radius 1,500. Top: trailing minus leading by objective (headline = sign agrees at radii 1,000/1,500/2,000 and $q_{BH}<0.05$). Bottom: one pooled contrast under three estimators.}
\label{tab:rlayer}
\setlength{\tabcolsep}{4pt}
\begin{adjustbox}{max width=\linewidth}
\begin{tabular}{llrrrc}
\toprule
Objective / estimator & Quantity & T$-$L (pp) & 95\% CI & $q_{BH}$ & Head. \\
\midrule
Elemental dragon & pit empty & +9.1 & [$-$5.6, 22.8] & 0.584 &  \\
 & mean $k$ & $-$0.13 & [$-$0.43, 0.15] & 0.769 &  \\
Atakhan & pit empty & +16.4 & [1.4, 32.3] & 0.213 &  \\
 & mean $k$ & $-$0.37 & [$-$0.78, 0.00] & 0.269 &  \\
Baron Nashor & pit empty & +19.7 & [11.4, 29.9] & 0.000 & \checkmark \\
 & mean $k$ & $-$0.39 & [$-$0.64, $-$0.14] & 0.000 & \checkmark \\
\midrule
TAGNet-style team mean & $\geq$3 at pit & $-$6.4 & [$-$11.2, $-$0.1] & $p$=0.046 &  \\
pooled (exposure-weighted) & $\geq$3 at pit & $-$4.1 & [$-$8.2, 0.0] & $p$=0.052 &  \\
within-team (teams seen in both states) & $\geq$3 at pit & $-$6.1 & [$-$9.9, 0.5] & $p$=0.078 &  \\
\bottomrule
\end{tabular}
\end{adjustbox}
\end{table}

%% file: tables/tab_selection.tex
\begin{table}[t]
\centering
\small
\caption{\textbf{No finite $K\leq8$ is selected on the V layer.} HASSE model selection. V layer (13 snapshots, 2,994 setups) and R layer (3 exact frames); held-out log-likelihood per setup over five series-grouped folds.}
\label{tab:selection}
\setlength{\tabcolsep}{4pt}
\begin{adjustbox}{max width=\linewidth}
\begin{tabular}{lrrrrrr}
\toprule
$K$ & Params & Held-out/setup & BIC & ICL & R: held-out & R: BIC \\
\midrule
1 & 76 & -92.85 & 556,150 & 556,150 & -22.48 & 73,517 \\
2 & 154 & -87.16 & 522,326 & 523,092 & -21.18 & 69,614 \\
3 & 235 & -83.83 & 502,918 & 503,899 & -20.70 & 68,225 \\
4 & 318 & -82.02 & 492,481 & 493,618 & -20.18 & 66,849 \\
5 & 403 & -80.82 & 485,443 & 487,048 & -19.87 & 66,306 \\
6 & 490 & -80.00 & 480,902 & 482,704 & -19.75 & 66,484 \\
7 & 579 & -79.59 & 478,926 & 481,145 & -19.68 & 66,689 \\
8 & 670 & -79.21 & 476,828 & 479,225 & -19.62 & 66,798 \\
\bottomrule
\end{tabular}
\end{adjustbox}
\end{table}

%% file: tables/tab_patterns.tex
\begin{table}[t]
\centering
\small
\caption{\textbf{Six patterns separate by objective class and by how they gather.} The pattern dictionary ($K=6$, V layer). Centre path $\mu_{c,t}$ at four of the 13 snapshots (T top, J jungle, M mid, B bot, S support; bars separate blocks), rigidity $\theta$ and inertia $\iota$ (per Hasse step) and the mixture weight in each objective class. Switching probability $\xi=0.137$.}
\label{tab:patterns}
\setlength{\tabcolsep}{4pt}
\begin{adjustbox}{max width=\linewidth}
\begin{tabular}{lrrrrrllll}
\toprule
 & $\theta$ & $\iota$ & $w_{river}$ & $w_{Atak.}$ & $w_{late}$ & $s{-}90$ & $s{-}30$ & $s$ & $s{+}60$ \\
\midrule
P0 & 1.00 & 0.90 & 0.14 & 0.33 & 0.18 & \texttt{T|J|M|B|S} & \texttt{T|J|M|B|S} & \texttt{BS|T|J|M} & \texttt{BS|T|J|M} \\
P1 & 1.28 & 1.13 & 0.25 & 0.01 & 0.00 & \texttt{BS|T|J|M} & \texttt{BS|T|J|M} & \texttt{BS|T|J|M} & \texttt{BS|T|J|M} \\
P2 & 0.98 & 0.86 & 0.23 & 0.29 & 0.16 & \texttt{T|J|M|B|S} & \texttt{BS|T|J|M} & \texttt{BS|T|J|M} & \texttt{JBS|T|M} \\
P3 & 0.56 & 0.72 & 0.10 & 0.11 & 0.38 & \texttt{JBS|T|M} & \texttt{JBS|T|M} & \texttt{JMBS|T} & \texttt{JBS|T|M} \\
P4 & 1.17 & 1.16 & 0.19 & 0.01 & 0.01 & \texttt{BS|T|J|M} & \texttt{T|J|M|B|S} & \texttt{BS|T|J|M} & \texttt{BS|T|J|M} \\
P5 & 0.73 & 0.65 & 0.09 & 0.26 & 0.27 & \texttt{JBS|T|M} & \texttt{JBS|T|M} & \texttt{JMBS|T} & \texttt{JMBS|T} \\
\bottomrule
\end{tabular}
\end{adjustbox}
\end{table}

%% file: tables/tab_hasse_tagnet.tex
\begin{table}[t]
\centering
\small
\caption{\textbf{HASSE refits agree where TAGNet seeds do not.} HASSE against TAGNet ported to the same V-layer tracks (TAGNet uses the 2,510 setups with at least 6 of 7 pre-spawn snapshots). HASSE rows are five independent fits, each the best of ten restarts; TAGNet rows are ten single runs per variant, with and without its context vector.}
\label{tab:hassetagnet}
\setlength{\tabcolsep}{4pt}
\begin{adjustbox}{max width=\linewidth}
\begin{tabular}{lll}
\toprule
 & HASSE ($K=6$) & TAGNet port \\
\midrule
Setups used & 2,994 (all) & 2,510 ($\geq$6 of 7 pre-spawn) \\
Choice of $K$ & held-out lik., BIC, ICL & fixed $K=6$ (K-Means) \\
Repeat-fit ARI & 0.890 (restarts) & 0.421 / 0.319 (seeds) \\
Series-bootstrap ARI & 0.819 (q05 0.756; whole model) & 0.813 (q05 0.608; K-Means only) \\
Agreement with counting baseline & -- & ARI 0.086 \\
Cramér's V with gold (clean) & 0.330 & 0.337 \\
Cramér's V with gold (context fed in) & 0.331 & 0.494 ($\times$1.24) \\
Parameters with units & $\theta,\iota$ (Hasse steps), $\xi$, zones & 64-d embedding \\
`only gather' test $H_{mono}$ & $\Delta$ held-out 0.027/setup (4/5 folds) & -- \\
\bottomrule
\end{tabular}
\end{adjustbox}
\end{table}

%% file: tables/tab_fights.tex
\begin{table}[t]
\centering
\small
\caption{\textbf{The larger side loses fewer players than the square law predicts.} Fight laws. Next death: logit of the next death falling on the larger side on $\ln(n_{large}/n_{small})$ and gold; Tullock: fight win on $\ln(k_i/k_j)$ and $\ln(\bar g_i/\bar g_j)$. Counts within 2,000 units 1 s before a kill; `tracked+V5' adds V5 killers and assisters (endogenous; sensitivity).}
\label{tab:fights}
\setlength{\tabcolsep}{4pt}
\begin{adjustbox}{max width=\linewidth}
\begin{tabular}{lllrrr}
\toprule
Model & Counts & Parameter & Est. & 95\% CI & $n$ \\
\midrule
Next death & tracked+V5 & intercept & $-$0.945 & [$-$1.162, $-$0.715] & 3,877 \\
Next death & tracked+V5 & $b_n$ (square $-1$, linear 0) & $-$2.327 & [$-$2.700, $-$1.959] & 3,877 \\
Next death & tracked+V5 & $b_g$ (gold) & $-$4.652 & [$-$5.901, $-$3.430] & 3,877 \\
Tullock fight & tracked+V5 & $\gamma_n$ (numbers) & 4.719 & [4.245, 5.272] & 2,139 \\
Tullock fight & tracked+V5 & $\gamma_g$ (gold) & 11.276 & [9.636, 13.097] & 2,139 \\
Tullock fight & tracked+V5 & $\gamma_{side}$ & $-$0.059 & [$-$0.199, 0.080] & 2,139 \\
Tullock fight & tracked+V5 & $\gamma_n/\gamma_g$ & 0.419 & [0.359, 0.499] & 2,139 \\
\midrule
Next death & tracked & intercept & $-$0.188 & [$-$0.400, 0.019] & 2,683 \\
Next death & tracked & $b_n$ (square $-1$, linear 0) & $-$0.566 & [$-$0.839, $-$0.286] & 2,683 \\
Next death & tracked & $b_g$ (gold) & $-$7.034 & [$-$8.138, $-$6.002] & 2,683 \\
Tullock fight & tracked & $\gamma_n$ (numbers) & 0.795 & [0.630, 0.952] & 1,481 \\
Tullock fight & tracked & $\gamma_g$ (gold) & 10.936 & [9.205, 12.830] & 1,481 \\
Tullock fight & tracked & $\gamma_{side}$ & $-$0.029 & [$-$0.154, 0.093] & 1,481 \\
Tullock fight & tracked & $\gamma_n/\gamma_g$ & 0.073 & [0.054, 0.095] & 1,481 \\
\bottomrule
\end{tabular}
\end{adjustbox}
\end{table}

%% file: tables/tab_shapley.tex
\begin{table}[t]
\centering
\small
\caption{\textbf{Curvature pulls trailing teams towards contesting; team behaviour pulls them away.} Decomposition of the trailing-minus-leading commitment gap (pp) and the named tests (theory direction in brackets).}
\label{tab:shapley}
\setlength{\tabcolsep}{4pt}
\begin{adjustbox}{max width=\linewidth}
\begin{tabular}{lrrr}
\toprule
Quantity & Est. & 95\% CI & $P(>0)$ \\
\midrule
observed gap & $-$19.19 & [$-$23.88, $-$15.14] & 0.00 \\
model gap (all channels) & $-$4.70 & [$-$6.71, $-$3.97] & 0.00 \\
model gap (channels off) & $-$5.15 & [$-$6.42, $-$4.22] & 0.00 \\
\midrule
Shapley: Lanchester & +0.15 & [0.05, 0.25] & 1.00 \\
Shapley: curvature of $W$ & +1.50 & [1.14, 2.10] & 1.00 \\
Shapley: second front & +0.30 & [0.21, 0.58] & 1.00 \\
Shapley: behavioural & $-$1.51 & [$-$3.42, $-$1.04] & 0.00 \\
\midrule
$H_{Simpson}$ raw gap [$<$ 0 (trailing contest less)] & $-$0.192 & [$-$0.237, $-$0.147] & $p$=0.000 \\
$H_{Simpson}$ matched (logit) [$\geq$ 0] & +0.240 & [$-$0.354, 0.778] & $p$=0.464 \\
$H_{trade}$ slope [$<$ 0] & +0.073 & [$-$0.026, 0.175] & $p$=0.156 \\
\bottomrule
\end{tabular}
\end{adjustbox}
\end{table}

%% file: tables/tab_vocab.tex
\begin{table}[t]
\centering
\small
\caption{\textbf{Formation vocabularies on the R layer, by objective and time.} Formation vocabulary on the R layer (five alive, risk set). Rarefied = expected distinct partitions at the smallest $n$.}
\label{tab:vocab}
\setlength{\tabcolsep}{4pt}
\begin{adjustbox}{max width=\linewidth}
\begin{tabular}{lrrrrr}
\toprule
Objective & Time & $n$ & Partitions & Shapes & Rarefied \\
\midrule
Atakhan & $s-60$ & 261 & 35 & 7 & 24.3 \\
Atakhan & $s$ & 283 & 39 & 7 & 23.6 \\
Atakhan & $s+60$ & 198 & 36 & 7 & 27.0 \\
Baron & $s-60$ & 239 & 41 & 7 & 27.8 \\
Baron & $s$ & 256 & 37 & 7 & 25.3 \\
Baron & $s+60$ & 200 & 37 & 7 & 27.4 \\
Dragon & $s-60$ & 304 & 21 & 5 & 12.8 \\
Dragon & $s$ & 317 & 18 & 4 & 11.8 \\
Dragon & $s+60$ & 266 & 20 & 6 & 14.3 \\
Grubs & $s-60$ & 294 & 26 & 6 & 15.8 \\
Grubs & $s$ & 317 & 22 & 6 & 14.1 \\
Grubs & $s+60$ & 17 & 9 & 4 & -- \\
Herald & $s-60$ & 289 & 24 & 6 & 16.8 \\
Herald & $s$ & 286 & 31 & 7 & 21.2 \\
Herald & $s+60$ & 90 & 23 & 7 & 23.0 \\
\bottomrule
\end{tabular}
\end{adjustbox}
\end{table}

%% file: tables/tab_hasse_robust.tex
\begin{table}[t]
\centering
\small
\caption{\textbf{$K=6$ is best among $K=3,\dots,6$ on the 7-snapshot grid and at both radii.} Discovery-layer robustness: held-out log-likelihood per setup for $K=3..6$ and the $K$ chosen by each criterion.}
\label{tab:hasserobust}
\setlength{\tabcolsep}{4pt}
\begin{adjustbox}{max width=\linewidth}
\begin{tabular}{lrrrrrrr}
\toprule
Layer & $K{=}3$ & 4 & 5 & 6 & best held-out & BIC & ICL \\
\midrule
main (13 snapshots, radius 1,900) & -83.83 & -82.02 & -80.82 & -80.00 & 6 & 6 & 6 \\
7 snapshots $s{-}90..s$ & -50.89 & -49.89 & -49.20 & -48.70 & 6 & 6 & 6 \\
formation radius 1,500 & -86.76 & -85.13 & -83.27 & -82.11 & 6 & 6 & 6 \\
formation radius 2,500 & -78.85 & -77.00 & -76.01 & -75.51 & 6 & 6 & 6 \\
\bottomrule
\end{tabular}
\end{adjustbox}
\end{table}

%% file: tables/tab_context.tex
\begin{table}[t]
\centering
\scriptsize
\caption{\textbf{Context effects on pattern occupancy are mostly small.} Soft-label multinomial logit of pattern occupancy (reference pattern 5), coefficients per SD; $^*$ = 95\% interval excludes 0 (500 series-block resamples).}
\label{tab:context}
\setlength{\tabcolsep}{4pt}
\begin{adjustbox}{max width=\linewidth}
\begin{tabular}{lrrrrrrrr}
\toprule
 & gold k & blue & minute & own soul pt & opp soul pt & elo diff & cls1 & cls2 \\
\midrule
P0 & -0.19$^*$ & -0.00 & -0.80$^*$ & -0.03 & -0.06 & -0.07$^*$ & +0.13$^*$ & +0.18$^*$ \\
P1 & -0.12 & -0.05 & -4.65$^*$ & +0.10 & +0.25$^*$ & -0.01 & +0.10 & +1.11$^*$ \\
P2 & -0.16$^*$ & -0.05 & -0.51$^*$ & -0.09$^*$ & -0.04 & -0.08$^*$ & -0.06 & -0.10$^*$ \\
P3 & -0.11$^*$ & -0.05$^*$ & +0.22$^*$ & -0.10$^*$ & -0.06 & -0.08$^*$ & -0.16$^*$ & +0.09 \\
P4 & -0.36$^*$ & +0.01 & -3.59$^*$ & +0.42$^*$ & +0.32$^*$ & +0.03 & -0.59$^*$ & +0.23 \\
\bottomrule
\end{tabular}
\end{adjustbox}
\end{table}

%% file: tables/tab_wfit.tex
\begin{table}[t]
\centering
\scriptsize
\caption{\textbf{The win surface fits out of fold in every game phase.} Out-of-fold fit of the outcome-only win surface $W$ by game phase.}
\label{tab:wfit}
\setlength{\tabcolsep}{4pt}
\begin{adjustbox}{max width=\linewidth}
\begin{tabular}{lrrrrrrr}
\toprule
$W$ specification & LL 1--13 & 14--24 & 25+ & Brier 1--13 & 14--24 & 25+ & max cal. dev. (pp) \\
\midrule
pooled, rule-specific objective terms (main) & 0.625 & 0.439 & 0.317 & 0.218 & 0.143 & 0.100 & 1.3 \\
same-rule games only & 0.657 & 0.455 & 0.288 & 0.233 & 0.150 & 0.093 & 3.7 \\
all games, shared terms & 0.624 & 0.439 & 0.319 & 0.218 & 0.143 & 0.101 & 1.3 \\
\bottomrule
\end{tabular}
\end{adjustbox}
\end{table}

%% file: tables/tab_qre_variants.tex
\begin{table}[t]
\centering
\scriptsize
\caption{\textbf{All Shapley signs agree with the main fit in the six non-degenerate variants.} LB-QRE specification variants. All Shapley signs agree with the main fit in the six non-degenerate variants; the two no-cost variants collapse to $\lambda=0$.}
\label{tab:qrevariants}
\setlength{\tabcolsep}{4pt}
\begin{adjustbox}{max width=\linewidth}
\begin{tabular}{lrrrrrrrrrr}
\toprule
Variant & $n$ & $\ell$ & $\lambda$ & $\nu$ & Obs. gap & Model gap & $\phi_L$ & $\phi_C$ & $\phi_F$ & $\phi_B$ \\
\midrule
main & 1,454 & -8,343 & 4.30 & 0.018 & $-$0.192 & $-$0.047 & 0.0015 & 0.0150 & 0.0030 & $-$0.0151 \\
drop held/interp. & 762 & -4,264 & 12.06 & 0.017 & $-$0.148 & $-$0.042 & 0.0047 & 0.0211 & 0.0092 & $-$0.0117 \\
V5 fight law & 1,454 & -8,318 & 4.78 & 0.121 & $-$0.192 & $-$0.035 & 0.0126 & 0.0137 & 0.0032 & $-$0.0130 \\
no costs (original spec.) & 1,454 & -9,254 & 0.00 & 0.993 & $-$0.192 & $-$0.027 & 0.0000 & 0.0000 & 0.0000 & $-$0.0006 \\
no costs + V5 law & 1,454 & -9,254 & 0.00 & 1.006 & $-$0.192 & $-$0.027 & 0.0000 & 0.0000 & 0.0000 & $-$0.0006 \\
pit radius 1,000 & 1,451 & -7,732 & 3.14 & 0.027 & $-$0.112 & $-$0.031 & 0.0008 & 0.0072 & 0.0014 & $-$0.0106 \\
pit radius 2,000 & 1,457 & -8,822 & 5.52 & 0.017 & $-$0.241 & $-$0.064 & 0.0015 & 0.0235 & 0.0054 & $-$0.0194 \\
`result' values & 1,454 & -8,386 & 2.91 & 0.024 & $-$0.192 & $-$0.045 & 0.0016 & 0.0151 & 0.0014 & $-$0.0175 \\
$\nu=1$ ($H_{winmax}$) & 1,454 & -8,403 & 0.73 & 1.000 & $-$0.192 & $-$0.043 & 0.0004 & 0.0075 & 0.0005 & $-$0.0168 \\
\bottomrule
\end{tabular}
\end{adjustbox}
\end{table}

%% file: tables/tab_team_params.tex
\begin{table}[t]
\centering
\scriptsize
\caption{\textbf{Team parameters come with wide intervals.} Team LB-QRE parameters with 95\% series-block bootstrap intervals (shrunk towards the global values).}
\label{tab:teamparams}
\setlength{\tabcolsep}{4pt}
\begin{adjustbox}{max width=\linewidth}
\begin{tabular}{lrlllll}
\toprule
Team & Contests & $\lambda$ & $\kappa$ & $\psi$ & $\omega_k$ & $\omega_m$ \\
\midrule
T1 & 421 & 3.3 [1.8, 6.3] & 1.73 [1.44, 2.07] & 1.39 [0.67, 2.55] & 0.21 [-0.01, 0.36] & 0.65 [0.56, 0.78] \\
Gen.G & 288 & 3.7 [2.2, 7.3] & 1.68 [1.48, 1.87] & 1.84 [1.46, 2.37] & 0.04 [-0.08, 0.14] & 0.78 [0.68, 0.93] \\
Anyone's Legend & 281 & 2.8 [1.0, 7.8] & 2.20 [1.81, 2.81] & 1.52 [0.92, 2.34] & 0.23 [0.07, 0.37] & 0.64 [0.53, 0.87] \\
G2 Esports & 265 & 6.2 [3.2, 12.8] & 1.75 [1.42, 2.17] & 2.35 [1.51, 3.45] & 0.16 [-0.04, 0.34] & 0.76 [0.57, 1.22] \\
Bilibili Gaming & 252 & 2.6 [1.3, 6.5] & 2.06 [1.81, 2.41] & 1.63 [0.55, 3.04] & 0.09 [-0.14, 0.36] & 0.58 [0.48, 0.76] \\
CTBC Flying Oyster & 220 & 6.0 [2.8, 12.9] & 1.92 [1.53, 2.49] & 1.64 [0.88, 2.49] & 0.20 [0.08, 0.38] & 0.81 [0.69, 1.23] \\
FlyQuest & 177 & 4.5 [1.9, 9.0] & 1.96 [1.35, 2.30] & 2.26 [1.36, 3.41] & 0.20 [-0.01, 0.50] & 0.63 [0.50, 0.82] \\
Movistar KOI & 164 & 9.9 [3.6, 16.8] & 1.74 [1.25, 2.38] & 2.50 [1.54, 3.66] & 0.24 [-0.01, 0.46] & 0.78 [0.57, 1.03] \\
KT Rolster & 150 & 4.0 [2.9, 7.5] & 1.47 [0.95, 1.90] & 1.65 [0.93, 2.47] & 0.04 [-0.08, 0.23] & 0.81 [0.56, 1.09] \\
Top Esports & 124 & 10.3 [6.0, 16.3] & 1.31 [0.53, 2.13] & 2.15 [1.52, 3.37] & 0.22 [-0.06, 0.56] & 0.87 [0.65, 1.09] \\
GAM Esports & 115 & 5.1 [2.3, 13.3] & 1.55 [1.06, 2.77] & 2.56 [1.67, 3.26] & 0.27 [0.14, 0.77] & 0.63 [0.50, 1.02] \\
FURIA & 92 & 4.3 [3.8, 6.7] & 2.09 [1.78, 2.32] & 1.32 [0.99, 2.27] & 0.36 [0.22, 0.42] & 0.64 [0.61, 0.85] \\
Hanwha Life Esports & 91 & 2.3 [1.9, 4.8] & 1.56 [1.32, 2.06] & 1.95 [1.41, 2.56] & 0.15 [-0.05, 0.37] & 0.68 [0.51, 0.93] \\
Vivo Keyd Stars & 59 & 3.9 [1.9, 11.2] & 2.99 [2.11, 3.54] & 1.32 [0.06, 2.87] & 0.57 [0.29, 0.96] & 0.50 [0.38, 0.83] \\
Team Secret Whales & 53 & 11.8 [4.3, 21.7] & 1.45 [0.82, 2.14] & 1.90 [0.87, 3.21] & 0.26 [0.06, 0.46] & 0.86 [0.60, 1.35] \\
100 Thieves & 44 & 2.7 [1.8, 20.8] & 1.62 [0.61, 2.63] & 2.10 [1.38, 3.04] & 0.35 [0.17, 0.84] & 0.69 [0.59, 0.91] \\
PSG Talon & 40 & 23.9 [4.6, 46.1] & 2.00 [1.36, 2.35] & 1.66 [1.35, 2.42] & 0.55 [0.26, 1.10] & 1.41 [0.75, 1.76] \\
Fnatic & 36 & 14.3 [3.3, 28.3] & 1.94 [1.37, 2.26] & 2.69 [1.85, 3.66] & 0.35 [0.20, 0.77] & 1.13 [0.76, 1.35] \\
Invictus Gaming & 36 & 3.4 [2.4, 7.0] & 1.56 [1.50, 1.92] & 1.18 [0.58, 2.33] & 0.28 [0.20, 0.37] & 0.77 [0.72, 0.86] \\
\bottomrule
\end{tabular}
\end{adjustbox}
\end{table}

%% file: tables/tab_usage.tex
\begin{table}[t]
\centering
\scriptsize
\caption{\textbf{Split-half reliability of team pattern usage is low.} Pattern usage by team ($K=6$). Intervals in Fig.~\ref{fig:fingerprints}; split-half reliability is low.}
\label{tab:usage}
\setlength{\tabcolsep}{4pt}
\begin{adjustbox}{max width=\linewidth}
\begin{tabular}{lrrrrrrrr}
\toprule
Team & Setups & P0 & P1 & P2 & P3 & P4 & P5 & TV to league \\
\midrule
T1 & 394 & 0.21 & 0.20 & 0.16 & 0.13 & 0.09 & 0.21 & 0.065 \\
Gen.G & 266 & 0.20 & 0.18 & 0.18 & 0.13 & 0.09 & 0.21 & 0.038 \\
Anyone's Legend & 260 & 0.24 & 0.15 & 0.21 & 0.14 & 0.08 & 0.18 & 0.043 \\
G2 Esports & 258 & 0.18 & 0.18 & 0.18 & 0.14 & 0.11 & 0.22 & 0.062 \\
Bilibili Gaming & 244 & 0.18 & 0.17 & 0.20 & 0.15 & 0.09 & 0.22 & 0.042 \\
CTBC Flying Oyster & 208 & 0.24 & 0.16 & 0.17 & 0.14 & 0.10 & 0.19 & 0.056 \\
FlyQuest & 165 & 0.22 & 0.16 & 0.19 & 0.16 & 0.10 & 0.18 & 0.042 \\
Movistar KOI & 161 & 0.19 & 0.15 & 0.24 & 0.14 & 0.07 & 0.21 & 0.048 \\
KT Rolster & 147 & 0.19 & 0.21 & 0.15 & 0.12 & 0.09 & 0.24 & 0.099 \\
Top Esports & 121 & 0.22 & 0.19 & 0.19 & 0.12 & 0.10 & 0.19 & 0.051 \\
GAM Esports & 113 & 0.17 & 0.19 & 0.18 & 0.18 & 0.09 & 0.19 & 0.072 \\
FURIA & 88 & 0.20 & 0.17 & 0.22 & 0.17 & 0.06 & 0.17 & 0.048 \\
Hanwha Life Esports & 82 & 0.20 & 0.13 & 0.20 & 0.18 & 0.06 & 0.22 & 0.070 \\
Vivo Keyd Stars & 55 & 0.17 & 0.17 & 0.24 & 0.19 & 0.07 & 0.16 & 0.090 \\
Team Secret Whales & 51 & 0.24 & 0.17 & 0.24 & 0.14 & 0.05 & 0.16 & 0.076 \\
100 Thieves & 40 & 0.16 & 0.18 & 0.23 & 0.13 & 0.08 & 0.22 & 0.063 \\
PSG Talon & 38 & 0.23 & 0.15 & 0.26 & 0.14 & 0.09 & 0.13 & 0.092 \\
Fnatic & 36 & 0.32 & 0.10 & 0.28 & 0.10 & 0.05 & 0.16 & 0.200 \\
Invictus Gaming & 31 & 0.16 & 0.20 & 0.15 & 0.14 & 0.10 & 0.24 & 0.104 \\
\bottomrule
\end{tabular}
\end{adjustbox}
\end{table}

%% file: tables/tab_faker.tex
\begin{table}[t]
\centering
\small
\caption{\textbf{Faker arrives earlier but is no more central in the formation.} A3: Faker against other starting mid laners in pre-spawn snapshots ($n$ = 432 vs 2,545).}
\label{tab:faker}
\setlength{\tabcolsep}{4pt}
\begin{adjustbox}{max width=\linewidth}
\begin{tabular}{lrrrr}
\toprule
Measure & Faker & Other mids & Diff. & 95\% CI \\
\midrule
in pit block at $s$ & 0.039 & 0.031 & 0.008 & [$-$0.000, 0.018] \\
arrival rank & 2.877 & 2.993 & $-$0.116 & [$-$0.226, $-$0.027] \\
geodesic distance to pit & 5,563 & 6,033 & $-$469 & [$-$709, $-$275] \\
alone & 0.631 & 0.622 & 0.009 & [$-$0.039, 0.063] \\
in main block & 0.337 & 0.333 & 0.004 & [$-$0.042, 0.049] \\
MST degree & 1.481 & 1.474 & 0.007 & [$-$0.023, 0.041] \\
MST betweenness & 0.259 & 0.253 & 0.006 & [$-$0.011, 0.023] \\
\bottomrule
\end{tabular}
\end{adjustbox}
\end{table}

%% file: tables/tab_t1_state.tex
\begin{table}[t]
\centering
\small
\caption{\textbf{T1's behaviour by gold state matches the rest of the league.} A5: T1 and all other teams by gold state at the commitment time (commit = $k\geq3$; polarisation = $k\in\{0,5\}$).}
\label{tab:t1state}
\setlength{\tabcolsep}{4pt}
\begin{adjustbox}{max width=\linewidth}
\begin{tabular}{lrrrrr}
\toprule
Group & State & $n$ & Commit & Polar. & Mean $k$ \\
\midrule
T1 & trailing & 93 & 0.097 & 0.731 & 0.60 \\
others & trailing & 652 & 0.107 & 0.724 & 0.60 \\
T1 & even & 211 & 0.256 & 0.460 & 1.31 \\
others & even & 1205 & 0.217 & 0.501 & 1.14 \\
T1 & leading & 117 & 0.291 & 0.274 & 1.57 \\
others & leading & 628 & 0.299 & 0.277 & 1.64 \\
\bottomrule
\end{tabular}
\end{adjustbox}
\end{table}

%% file: tables/tab_ledger.tex
\begin{table}[t]
\centering
\scriptsize
\caption{\textbf{The ledger lists every spawn of the Worlds~2025 final.} A7: decision ledger of the Worlds 2025 final (T1 3--2 KT Rolster). Actions $(k,m,\tau)$ at the commitment time, blue's fight probability $\hat p$, break-even $\pi^*$ and the objective's value to blue (pts).}
\label{tab:ledger}
\setlength{\tabcolsep}{4pt}
\begin{adjustbox}{max width=\linewidth}
\begin{tabular}{rlrrrllrrrc}
\toprule
G & Spawn & Min & $W_B$ & GD$_B$ & Blue & Red & $\hat p$ & $\pi^*$ & $V_B$ & Took \\
\midrule
1 & Drag 1 & 5.0 & 0.54 & +21 & (2,1,0) & (0,1,0) & 0.51 & 0.38 & +5.8 & B \\
1 & Grubs 1 & 8.0 & 0.66 & -7 & (1,1,1) & (0,3,0) & 0.50 & 0.45 & +3.5 & B \\
1 & Drag 2 & 10.9 & 0.74 & +302 & (3,0,0) & (2,0,0) & 0.65 & 0.44 & +6.0 & R \\
1 & Herald 1 & 15.0 & 0.64 & +71 & (3,0,0) & (0,0,0) & 0.55 & 0.45 & +5.5 & B \\
1 & Drag 3 & 16.8 & 0.82 & +1,472 & (0,2,0) & (3,1,0) & 0.64 & 0.45 & +5.7 & R \\
1 & Atak 1 & 20.0 & 0.63 & +564 & (2,0,0) & (0,0,0) & 0.55 & 0.25 & +17.6 & R \\
1 & Drag 4 & 22.4 & 0.63 & +713 & (2,0,0) & (3,0,0) & 0.55 & 0.23 & +9.6 & B \\
1 & Baron 1 & 25.0 & 0.45 & -278 & (0,0,0) & (3,0,0) & 0.48 & 0.13 & +33.9 & R \\
1 & Drag 6 & 33.7 & 0.19 & -4,059 & (0,0,0) & (1,1,1) & 0.32 & 0.18 & +14.6 & R \\
1 & Baron 2 & 33.8 & 0.19 & -4,059 & (0,0,0) & (3,0,0) & 0.32 & 0.15 & +39.2 & -- \\
2 & Drag 1 & 5.0 & 0.65 & +616 & (0,1,0) & (0,1,0) & 0.74 & 0.40 & +5.2 & B \\
2 & Grubs 1 & 8.0 & 0.73 & +1,075 & (1,2,0) & (0,4,0) & 0.74 & 0.47 & +3.0 & B \\
2 & Drag 2 & 12.2 & 0.80 & +1,722 & (1,1,0) & (0,1,0) & 0.74 & 0.48 & +5.2 & B \\
2 & Herald 1 & 15.0 & 0.87 & +1,976 & (3,1,0) & (0,2,0) & 0.72 & 0.56 & +2.5 & B \\
2 & Drag 3 & 17.7 & 0.90 & +2,485 & (3,0,0) & (0,0,0) & 0.72 & 0.43 & +7.8 & R \\
2 & Atak 1 & 20.0 & 0.42 & -509 & (4,0,0) & (1,1,1) & 0.46 & 0.22 & +21.3 & R \\
2 & Drag 4 & 23.5 & 0.14 & -3,860 & (3,0,0) & (3,1,0) & 0.26 & 0.17 & +15.6 & R \\
2 & Baron 1 & 25.0 & 0.16 & -3,823 & (0,1,0) & (3,1,0) & 0.27 & 0.10 & +35.9 & R \\
2 & Drag 5 & 29.3 & 0.16 & -4,272 & (3,0,1) & (0,0,0) & 0.28 & 0.20 & +15.0 & R \\
2 & Baron 2 & 34.9 & 0.03 & -8,952 & (0,0,0) & (3,0,0) & 0.17 & 0.09 & +23.0 & R \\
2 & Drag 6 & 35.4 & 0.03 & -8,952 & (0,1,0) & (0,4,0) & 0.09 & 0.12 & +4.6 & -- \\
3 & Drag 1 & 5.0 & 0.70 & +804 & (0,1,0) & (2,1,0) & 0.80 & 0.42 & +4.7 & R \\
3 & Grubs 1 & 8.0 & 0.57 & +629 & (1,1,0) & (0,1,0) & 0.66 & 0.42 & +3.9 & B \\
3 & Drag 2 & 12.0 & 0.85 & +2,131 & (0,0,0) & (1,1,0) & 0.79 & 0.50 & +3.6 & R \\
3 & Herald 1 & 15.0 & 0.79 & +2,335 & (1,0,1) & (0,0,0) & 0.75 & 0.51 & +3.9 & B \\
3 & Drag 3 & 17.7 & 0.87 & +3,502 & (1,0,0) & (0,0,0) & 0.87 & 0.36 & +4.0 & B \\
3 & Atak 1 & 20.0 & 0.67 & +2,255 & (2,0,0) & (0,0,0) & 0.74 & 0.26 & +16.1 & B \\
3 & Baron 1 & 25.0 & 0.91 & +4,347 & (3,1,0) & (0,5,0) & 0.71 & 0.11 & +12.3 & B \\
3 & Drag 4 & 26.0 & 0.92 & +4,807 & (3,0,0) & (0,2,0) & 0.76 & 0.25 & +3.1 & B \\
3 & Drag 5 & 31.8 & 0.85 & +3,332 & (3,1,0) & (0,4,0) & 0.66 & 0.39 & +8.7 & B \\
3 & Drag 6 & 37.3 & 0.98 & +9,905 & (3,0,0) & (1,0,0) & 0.83 & 0.51 & +9.1 & B \\
3 & Baron 2 & 40.3 & 0.97 & +9,097 & (4,0,0) & (0,0,0) & 0.82 & 0.50 & +4.7 & B \\
4 & Drag 1 & 5.0 & 0.64 & +372 & (0,1,0) & (2,1,0) & 0.66 & 0.40 & +5.2 & B \\
4 & Grubs 1 & 8.0 & 0.50 & -482 & (1,1,1) & (0,1,0) & 0.38 & 0.41 & +4.0 & B \\
4 & Drag 2 & 10.8 & 0.73 & +353 & (2,2,0) & (1,1,0) & 0.56 & 0.43 & +6.2 & B \\
4 & Herald 1 & 15.0 & 0.75 & -58 & (2,1,0) & (0,3,0) & 0.49 & 0.49 & +4.4 & B \\
4 & Drag 3 & 16.5 & 0.83 & +422 & (2,1,0) & (4,0,0) & 0.55 & 0.41 & +11.4 & B \\
4 & Atak 1 & 20.0 & 0.83 & +504 & (4,0,0) & (4,0,0) & 0.54 & 0.29 & +9.6 & B \\
4 & Drag 4 & 22.1 & 0.92 & +1,196 & (4,0,0) & (3,0,0) & 0.59 & 0.20 & +17.2 & B \\
4 & Baron 1 & 25.0 & 0.99 & +4,488 & (3,1,0) & (2,0,0) & 0.81 & 0.07 & +3.0 & B \\
5 & Drag 1 & 5.0 & 0.49 & -320 & (4,1,0) & (0,1,0) & 0.36 & 0.37 & +6.0 & B \\
5 & Grubs 1 & 8.0 & 0.57 & -712 & (1,1,1) & (1,1,0) & 0.33 & 0.42 & +3.9 & R \\
5 & Drag 2 & 11.2 & 0.48 & -1,095 & (0,1,0) & (2,1,0) & 0.32 & 0.36 & +8.6 & R \\
5 & Herald 1 & 15.0 & 0.28 & -1,800 & (3,0,0) & (0,1,0) & 0.29 & 0.33 & +5.3 & B \\
5 & Drag 3 & 16.7 & 0.23 & -2,446 & (3,0,0) & (0,1,0) & 0.26 & 0.28 & +8.8 & B \\
5 & Atak 1 & 20.0 & 0.34 & -2,489 & (0,0,0) & (3,0,0) & 0.30 & 0.21 & +21.2 & R \\
5 & Drag 4 & 22.4 & 0.45 & -2,431 & (0,0,0) & (3,0,1) & 0.12 & 0.22 & +21.2 & R \\
5 & Baron 1 & 25.0 & 0.04 & -6,216 & (0,4,0) & (0,5,0) & 0.17 & 0.06 & +24.2 & R \\
5 & Drag 5 & 27.6 & 0.10 & -5,733 & (1,0,1) & (3,0,0) & 0.21 & 0.17 & +11.8 & R \\
5 & Drag 6 & 34.6 & 0.03 & -10,456 & (0,0,0) & (1,0,1) & 0.13 & 0.12 & +4.7 & R \\
5 & Baron 2 & 34.9 & 0.03 & -10,456 & (0,0,0) & (3,0,0) & 0.13 & 0.09 & +22.7 & -- \\
\bottomrule
\end{tabular}
\end{adjustbox}
\end{table}